\documentclass[lettersize,journal]{IEEEtran}
\usepackage{amsmath,amsfonts}
\usepackage{graphicx}
\usepackage{float}
\usepackage{algorithm}
\usepackage{array}
\usepackage{textcomp}
\usepackage{stfloats}
\usepackage{url}
\usepackage{verbatim}
\usepackage{graphicx}
\usepackage{cite}
\usepackage{amsmath, amssymb, epsfig, graphicx, bm}
\usepackage{epstopdf}
\usepackage{cases}
\usepackage{float}
\usepackage[justification=centering]{caption}
\usepackage{graphicx}
\usepackage{color}
\usepackage{epstopdf}
\usepackage{amsthm}
\usepackage{amsfonts}
\usepackage{cite, url}
\usepackage[all,cmtip]{xy}
\usepackage{color,hyperref}
\usepackage[noend]{algpseudocode}
\usepackage{flushend}
\usepackage[normalem]{ulem} % to use \sout

\newtheorem{theorem}{Theorem}
\newcommand{\E}{{\mathbb{E}}}
\newcommand{\R}{{\mathbb{R}}}
\newcommand{\RR}{{\mathcal{R}}}
\newcommand{\B}{{\mathcal{B}}}

\newtheorem{Definition}{Definition}

\newtheorem{assumption}{Assumption}
\newtheorem{lemma}{Lemma}
\newtheorem{remark}{Remark}

\newcommand{\vx}{\bm{x}}

\newcommand{\N}{\mathcal{N}}

\newcommand{\cD}{\mathcal{D}}

\newcommand{\cS}{\mathcal{S}}
\newcommand{\cL}{\mathcal{L}}

\newcommand{\A}{\mathcal{A}}
\newcommand{\lp}{\left(}
\newcommand{\rp}{\right)}

\newcommand{\lnorm}{\left\|}
\newcommand{\rnorm}{\right\|}

\newcommand{\smallCE}[1]{$<$1e-07}

\begin{document}

\title{Generalization Error Matters in Decentralized Learning Under Byzantine Attacks}

\author{Haoxiang Ye and Qing Ling
        % <-this % stops a space
%\thanks{Manuscript received  XX, 2023; revised  XX, XXXX.}
\thanks{Haoxiang Ye and Qing Ling are with the School of Computer Science and Engineering, Sun Yat-Sen University, Guangzhou, Guangdong, China 510006.
        (e-mail: yehx9@mail2.sysu.edu.cn; lingqing556@mail.sysu.edu.cn)}%
\thanks{Qing Ling (corresponding author) is supported in part by NSF China grants 12126610 \& 62373388, Guangdong Basic and Applied Basic Research Foun- dation grant 2021B1515020094 \& 2023B1515040025, R\&D project of Pazhou Lab (Huangpu) grant 2023K0606, and Guangdong Provincial Key Laboratory of Mathematical Foundations for Artificial Intelligence grant 2023B121201- 0001. A short and preliminary version of this paper has appeared in ICASSP 2024 \cite{ye2024ge}.}

}

% \markboth{IEEE Transactions on Information Forensics and Security, SUBMITTED}%


\maketitle

\begin{abstract}
Recently, decentralized learning has emerged as a popular peer-to-peer signal and information processing paradigm that enables model training across geographically distributed agents in a scalable manner, without the presence of any central server. When some of the agents are malicious (also termed as Byzantine), resilient decentralized learning algorithms are able to limit the impact of these Byzantine agents without knowing their number and identities, and have guaranteed optimization errors. However, analysis of the generalization errors, which are critical to implementations of the trained models, is still lacking. In this paper, we provide the first analysis of the generalization errors for a class of popular Byzantine-resilient decentralized stochastic gradient descent (DSGD) algorithms. Our theoretical results reveal that the generalization errors cannot be entirely eliminated because of the presence of the Byzantine agents, even if the number of training samples are infinitely large. Numerical experiments are conducted to confirm our theoretical results.
\end{abstract}


\section{Introduction}
\label{sec1}

\IEEEPARstart{D}{ue} to the demands of signal and information processing over vast data spreading across geographically distributed devices, distributed learning that respects data privacy has gained considerable attention from both academia and industry in recent years. According to the underlying communication topology, distributed learning can be categorized into two main types: central server-based federated learning \cite{McMahan2016,10180365,Bian2024}
% 10095595
and peer-to-peer decentralized learning \cite{Lian2017,9713700,9802673,chen2024}.
%  9414564
Among them, decentralized learning stands out as it is more scalable to the network size and free of the communication bottleneck caused by the central server. For decentralized learning, decentralized stochastic gradient descent (DSGD) is one of the most popular algorithms \cite{nedic2009distributed}. In DSGD, devices (also termed as agents) run local stochastic gradient descent steps, send the intermediate results to their neighbors while receive their neighbors' intermediate results, and aggregate the results using weighted averaging to yield the new iterates.

However, the message passing in decentralized learning is subject to various uncertainties: packet losses, communication delays, and even malicious attacks. Under such uncertainties, the weighted averaging steps in DSGD become vulnerable. We focus on handling malicious attacks in decentralized learning, and characterize them with the Byzantine attacks model \cite{lamport}. In the Byzantine attacks model, an unknown number of malicious agents (also termed as Byzantine agents) with unknown identities can send arbitrarily malicious messages, instead of the true intermediate results in the context of DSGD, to their neighbors. The main idea of handling Byzantine attacks is to aggregate the received messages with robust aggregation rules other than the vulnerable weighted averaging.

While various robust aggregation rules have been proposed to defend against Byzantine attacks within central server-based federated learning \cite{yin2018byzantine,chen2017distributed,blanchard2017machine,karimireddy2021learning,xia2019faba,rsa,dong2024}, it has been already depicted that many of them exhibit significant performance degradation in decentralized learning, both theoretically and empirically \cite{wu2022byzantine}. Provably effective robust aggregation rules for decentralized learning include trimmed mean (TM) \cite{fang2022bridge}, iterative outlier scissor (IOS) \cite{wu2022byzantine} and self centered clipping (SCC)\cite{he2022byzantine}.
% In TM, each non-Byzantine agent removes a number of the largest and smallest elements in each dimension of the received messages, and then averages the rest. In IOS, each non-Byzantine agent maintains a trusted set of the received messages, iteratively removes messages that significantly deviate from the set's average, and then takes weighted averaging on the rest. In SCC, each non-Byzantine agent does not remove any received messages, but applies an clipping operator to limit the influence of those faraway from its own message, and then takes weighted averaging.
The optimization errors of Byzantine-resilient DSGD algorithms, which replace weighted averaging with these robust aggregation rules, are well-studied. Theoretical analyses have shown that the Byzantine-resilient DSGD algorithms can converge to bounded neighborhoods of an optimal solution \cite{wu2022byzantine,fang2022bridge,he2022byzantine,Ye2023}.

Despite the fruitful theoretical and empirical results on the optimization errors of Byzantine-resilient DSGD algorithms \cite{wu2022byzantine,fang2022bridge,he2022byzantine,Ye2023}, investigation of their generalization errors remains absent. The generalization error reflects the applicability of a trained model in the real world and is crucial for a learning algorithm. Some recent works have analyzed the generalization errors of attack-free DSGD algorithms \cite{sun2021stability,deng2023stability,zhu2022topology,bars2023improved,richards2020graph,taheri2023generalization} within the framework of algorithmic stability \cite{bousquet2002stability,hardt2016train}. However, these works do not take into account the presence of Byzantine attacks in their analyses. The generalization error of Byzantine-resilient decentralized learning remains an open question.

\subsection{Contributions}
Our contributions are summarized as follows.
\begin{itemize}

\item \textbf{First Analysis of Generalization Error for Byzantine-Resilient DSGD:} We provide the first-ever generalization error analysis for Byzantine-resilient decentralized learning. Our theoretical results reveal the negative impact of Byzantine attacks on the generalization error. Particularly, in the presence of Byzantine agents, the generalization error cannot vanish, even in an ideal scenario with an infinitely large number of training samples.
% \item \textbf{First Generalization Error Improvement Algorithm:} We introduce the first algorithm designed to enhance the generalization error of Byzantine-resilient decentralized learning algorithms. Our approach involves lightweight and innovative modifications based on the normalized gradient mechanism, which lead to significant improvements in generalization error.
% {\color{blue}\item \textbf{First Analysis of Worst-Case Generalization Error for Byzantine-Resilient DSGD:} We expand our analysis from the on-average generalization error to include the worst-case generalization error. Theoretical results regarding the worst-case generalization error underscore its reliance on the network topology.}

\item \textbf{Validation with Numerical Experiments:} To validate these theoretical findings, we conduct numerical experiments to explore the generalization abilities of Byzantine-resilient DSGD with different robust aggregation rules. The numerical results align with the conclusions drawn from our theoretical findings.
%The results of these experiments affirm that the proposed modifications indeed enhance the performance of various decentralized Byzantine-resilient algorithms.
\end{itemize}
\subsection{Related works}
\noindent\textbf{Byzantine-Resilience.} As a class of the most critical threats to distributed systems \cite{lamport}, Byzantine attacks significantly dis- courage the implementation of distributed learning. To defend against Byzantine attacks, a common strategy is using robust aggregation rules to replace the vulnerable mean or weighted mean aggregation rules, whether in sever-based federated lear- ning or in decentralized learning.

In server-based federated learning, numerous robust aggregation rules have been proposed, including those discussed in \cite{yin2018byzantine,chen2017distributed,blanchard2017machine,karimireddy2021learning,xia2019faba,rsa,dong2024}.
%Furthermore, \cite{liu2023byzantine,allouah2023fixing} have introduced additional mechanisms specifically aimed at enhancing the performance of these aggregation rules.
%In \cite{liu2023byzantine}, the authors propose splitting the input vectors and applying robust aggregation rules to each group of sub-vectors. In \cite{allouah2023fixing}, the authors suggest replacing each input vector with the average of its $|\RR|$ nearest neighbors, where $|\RR|$ is the number of non-Byzantine inputs.
However, some of them may fail in decentralized learning, as highlighted in \cite{wu2022byzantine}. Robust aggregation rules with guaranteed effectiveness for decentralized learning include TM \cite{fang2022bridge}, IOS \cite{wu2022byzantine} and SCC \cite{he2022byzantine}. In TM each non-Byzantine agent first
removes a number of the largest and smallest elements in each dimension of the received messages, and then averages the rest. In IOS each non-Byzantine agent maintains a trusted set of the received messages, iteratively removes messages that significantly deviate from the set's average, and then takes weighted averaging on the rest. In SCC each non-Byzantine agent does not remove any received messages, but applies an clipping operator to limit the influence of those faraway from its own message, and then takes weighted averaging.

\noindent\textbf{Generalization Error.}
Conventional analyses of learning algorithms typically focus on the optimization errors during the training processes. However, it is also of practical importance to consider the generalization errors of the trained models on unseen testing samples. Popular technical tools for analyzing the generalization errors include
Rademacher complexity \cite{sachs2023generalization}, probably approximately correct (PAC) learnability \cite{london2017pac,yang2019pac}, as well as those from information-theoretic perspectives \cite{xu2017information}. Another notable technical tool is algorithmic stability, which characterizes the change of the trained model when a single training sample is substituted. The concept of uniform stability, initially introduced by \cite{bousquet2002stability}, is subsequently applied by \cite{hardt2016train} to explore the generalization error of stochastic gradient descent (SGD). The work of \cite{zhang2022stability} establishes a lower bound of uniform stability and asserts that the upper bounds obtained in \cite{hardt2016train} are tight when the loss is convex.

The works of \cite{sun2021stability,deng2023stability,zhu2022topology,bars2023improved,taheri2023generalization,richards2020graph} have extended the above theoretical results from SGD to DSGD, thereby helping understand the generalization error in the context of decentralized learning.
%Some other works also investigate the generalization error in federated learning setting\cite{9520293,barnes2022improved}.
% Some works investigate the generalization error of central server-based federated learning \cite{hu2022generalization,wu2023federated,wu2023information,sun2023understanding}.
Among these works, \cite{sun2021stability} and \cite{deng2023stability} analyze the generalization errors of synchronous and asynchronous DSGD, respectively. The impact of the communication topology on the generalization error of DSGD is investigated in \cite{zhu2022topology}.
The generalization error is improved in \cite{bars2023improved,richards2020graph}, matching the bound established for single-agent algorithms.
% {\color{blue}Moreover, \cite{bars2023improved,richards2020graph} concluded that the generalization error of DSGD is independent of network topology. They argue that the topology-dependent generalization error bounds derived in \cite{sun2021stability,deng2023stability,zhu2022topology} are not tight. }
The work of \cite{taheri2023generalization} investigates the generalization error of DSGD when the training samples are
linearly separable. However, none of them considers the generalization error of Byzantine-resilient decentralized learning.

Compared to the short and preliminary conference version \cite{ye2024ge}, this paper has been extensively revised. The conference version only includes the generalization error analysis when the loss is strongly convex, whereas this paper has expanded the scope to include both convex and non-convex losses. Also,

\noindent we have included detailed proofs that highlight the challenges in analyzing the generalization error of Byzantine-resilient decentralized learning.
Last but not least, we have augmented this paper with additional numerical experiments and discussions on the numerical results.

\subsection{Paper Organization}
The rest of this paper is organized as follows. The problem of Byzantine-resilient decentralized learning is formulated in Section \ref{sec2}, taking into account the generalization error issue. A generic Byzantine-resilient decentralized SGD framework is reviewed in Section \ref{sec3}. We introduce the assumptions and the technical tool of uniform stability in Section \ref{sec-pre}, and present the derived generalization error bounds in Section \ref{sec-ge}. Numerical experiments are given in Section \ref{sec-num}, and conclusions are made in Section \ref{sec-con}.

\section{Problem Statement}
\label{sec2}
%\vspace{-0.5em}

%   Let $\cS=\{\xi_1,\cdots,\xi_N\}$ be a set of training samples drawn independently and identically from data distribution $\cD$.
%   To goal of machine learning is to learn a model $\vx$ to approximate the mapping between the input variable and output variable.
%   The loss of model $\vx$ with respect to the sample $\xi$ is donated by $f({\vx}; \xi)$. Then, the corresponding empirical and population risks of $\vx$ are defined as follows:
%  \begin{align*}
% F_\cS({\vx}):= \frac{1}{N}\sum_{i=1}^N f({\vx}; \xi_i)  \quad \text{and} \quad F({\vx}):=\E_{\xi\sim \mathcal{D}} f({\vx}; \xi)
%  \end{align*}
% Due to the data distribution $\mathcal{D}$ is usually unknown,  people turn to solve the empirical risk minimization (ERM) problem.
 %  \begin{align*}
 %  \min_{{\vx}\in\R^D} F_S({\vx}):= \frac{1}{N}\sum_{i=1}^N f({\vx}; \xi_i),
 % \end{align*}
We consider a fundamental signal and information processing task of training a model with a set of agents over a peer-to-peer communication network. The network is described by an undirected and connected graph, denoted as $\mathcal{G}=(\mathcal{N},\mathcal{E})$, where $\mathcal{N}=\left\{1,...,N\right\}$ is the set of agents and $\mathcal{E}$ is the set of edges. An edge $e = (m, n) \in \mathcal{E}$ stands for an undirected communication link between agents $m$ and $n$, enabling the two neighbors to transmit messages to each other. Some of the agents, whose number and identities are unknown, will transmit arbitrarily malicious messages during the learning process. We also call them as Byzantine agents \cite{lamport}. We define $\mathcal{R}$ as the set of non-Byzantine agents and $\mathcal{B}$ as the set of Byzantine agents, with $\N = \mathcal{R} \cup \mathcal{B}$. Each non-Byzantine agent $n$ draws training samples from a local data distribution denoted as $\cD_n$, with respect to a random variable $\xi_n \sim \cD_n$. The goal is to train a global model $\vx^* \in \R^d$ that minimizes the population risk defined as
\begin{equation}\label{1}
    F(\vx) := \frac{1}{|\RR|}\sum\limits_{n \in \RR}\mathbb{E}_{\xi_n \sim \cD_n } f\left(\vx; \xi_n\right),
\end{equation}
in which $f\left(\vx; \xi_n\right)$ is the loss of $\vx \in \R^d$ with respect to $\xi_n$. Note that we should not include the Byzantine agents in \eqref{1} since being able to transmit arbitrarily malicious messages during the learning process implies the ability of arbitrarily manipulating their training samples.

Since the true data distributions $\mathcal{D}_n$ are typically unknown, a common approach is to minimize the empirical risk over the union of the non-Byzantine agents' local datasets, denoted as $\cS := \cup_{n \in \mathcal{R}} \cS_{n}$, where $\cS_n$ $=\{\xi_{n,1},\cdots,\xi_{n,Z}\}$ is the local dataset of non-Byzantine agent $n \in \RR$ with size $Z$, as well as independently and identically sampled from $\cD_n$. Without loss of generality, we assume that all the local datasets are with the same size. The global empirical risk averages the local empirical risks of the non-Byzantine agents, in the form of
\begin{align}\label{erm}
F_{\cS}(\vx) \!& :=\! \frac{1}{|\RR|} \! \sum\limits_{n \in \RR} F_{\mathcal{S}_n}\!\left(\vx\right), \\
\text{with} ~ F_{\mathcal{S}_n}\!\left(\vx\right) \!&:=\! \frac{1}{Z} \! \sum_{i=1}^{Z}f\left(\vx; \xi_{n,i}\right). \notag
\end{align}
For later usage, we denote $\vx_{\cS}^*  \in \R^d$ as the minimizer of $F_{\cS}$.

Let us consider a stochastic decentralized learning algorithm $\cL$ applied to dataset $\cS$, yielding a model $\cL(\cS)$. The expected excess risk \cite{bottou2007tradeoffs} that characterizes the performance degradation of $\cL(\cS)$ compared to $\vx^*$, the minimizer of the population risk \eqref{1}, is given by
  \begin{equation*}
  \begin{aligned}
  &\E_{\cS,\mathcal{L}}[F(\mathcal{L}(\cS))-F({\vx}^*)]=\underbrace{\E_{\cS,\mathcal{L}}[F(\mathcal{L}(\cS))-F_{\cS}(\mathcal{L}(\cS))]}_{\textrm{generalization  error}}\\
  &+\underbrace{\E_{\cS,\mathcal{L}}[F_{\cS}(\mathcal{L}(\cS))-F_{\cS}(\vx^*_{\cS})]}_{\textrm{optimization error}}+\underbrace{\E_{\cS,\mathcal{L}}[F_{\cS}(\vx^*_{\cS})-F({\vx}^*)]}_{\textrm{$\leq 0$}}.
  \end{aligned}
  \end{equation*}
The last term is non-positive because: (i) $F_{\cS}({\vx}_{\cS}^*) \leq F_{\cS}({\vx}^*)$ as ${\vx}_{\cS}^*$ is the minimizer of $F_{\cS}$; (ii)
$\E_{\cS,\mathcal{L}} [F({\vx}^*)] = \E_{\cS,\mathcal{L}} [F_{\cS}({\vx}^*)]$ as $\cS_{n}$ is independently and identically sampled from $\cD_n$. Thus, the expected excess risk is upper-bounded by the summation of generalization and optimization errors. Although the optimization errors of Byzantine-resilient DSGD algorithms have been extensively studied \cite{wu2022byzantine,fang2022bridge,he2022byzantine,Ye2023}, the generalization errors still lack investigation. In the subsequent section, we provide the first-ever analysis of the generalization errors for a class of Byzantine-resilient DSGD algorithms.

\section{Byzantine-resilient DSGD}
\label{sec3}
We start from reviewing the conventional attack-free DSGD. At time $k$, each agent $n$ holds a local model $\vx^{k}_n \in \R^d$, draws a training sample $\xi_n^{k}$ from the local dataset $\cS_n$, calculates a stochastic gradient $\nabla f(\vx^{k}_n; \xi_n^{k})$, and runs a stochastic gradient descent step $\vx_n^{k+\frac{1}{2}} = \vx^k_n - \alpha^{k} \nabla f(\vx^{k}_n; \xi_n^{k})$, where $\alpha^k > 0$ is a step size. Then, it transmits $\vx_{n,m}^{k+\frac{1}{2}} = \vx_n^{k+\frac{1}{2}}$ to all neighbors $m \in \mathcal{N}_n$, in which $\mathcal{N}_n$ denotes the set of agent $n$'s neighbors. Upon receiving $\vx_{m,n}^{k+\frac{1}{2}}$ from all neighbors $m \in \mathcal{N}_n$, each agent $n$ undertakes a weighted averaging step $\vx^{k+1}_n=\sum_{m \in \N_n\cup\{n\}}$ $w_{nm}^\prime \vx_{m,n}^{k+\frac{1}{2}}$, within which the weights $w_{nm}^\prime \in [0,1]$ satisfies $\sum_{m\in \N_n \cup \{n\}}w_{nm}^\prime=1$.

However, at the presence of Byzantine agents, the weighted averaging step becomes highly vulnerable. Denote $\mathcal{R}_n$ and $\mathcal{B}_n$ as the sets of agent $n$'s non-Byzantine and Byzantine neighbors, respectively. Thus, $\vx^{k+1}_n=\sum_{m \in \mathcal{R}_n \cup \B_n \cup\{n\}}w_{nm}^\prime \vx_{m,n}^{k+\frac{1}{2}}$ can be arbitrary even when there is only one Byzantine agent. To counteract Byzantine attacks, Byzantine-resilient DSGD replaces the weighted averaging step by a robust aggregation rule. Denoting $\mathcal{A}_n$ as the robust aggregation rule used by non-Byzantine agent $n \in \mathcal{R}$, we outline Byzantine-resilient DSGD in Algorithm \ref{robust-DSGD}.

Following \cite{wu2022byzantine}, we investigate a class of robust aggregation rules $\{\mathcal{A}_n\}_{n\in \RR}$ that, for each non-Byzantine agent $n\in \RR$, are able to approximate a ``proper'' weighted average of non-Byzantine neighbors' local models and its own. The associated weights and the approximation ability are characterized as follows. Without loss of generality, we assume that the non-Byzantine agents are numbered from $1$ to $|\RR|$ thereafter.

%denoted as $\hat\vx_n$, within the average step.
%A superior aggregation rule $\mathcal{A}_n$ provides a more accurate estimation of $\hat\vx_n$, and we employ Definition \ref{definition:mixing-matrix} to objectively evaluate the performance of $\mathcal{A}_n$ as \cite{wu2022byzantine}.
%We hope that the aggregated result of $\A_n$ is closed to $\hat\vx^k$, and we use Definition \ref{definition:mixing-matrix} to evaluate the performance of $\A_n$.

\begin{Definition}[Virtual mixing matrix and contraction constant corresponding to $\{\mathcal{A}_n\}_{n \in \RR}$]
    \label{definition:mixing-matrix}
    Let us consider a matrix $W \!\in\! \mathbb{R}^{\vert \RR \vert \times \vert \RR \vert}$ whose $(n,m)$-th entry $w_{nm} \in [0, 1]$
    if $m \in \RR_n \cup\{n\}$ and $w_{nm} = 0$ if $m \in \RR$ but $m \notin \RR_n \cup\{n\}$, for $n \in \RR$.
    Further, $\sum_{m\in \RR_n \cup \{n\}}w_{nm}=1$. Define $\hat \vx_n := \sum_{m \in \RR_n\cup\{n\}}$ $w_{nm} \vx_{m,n}$. If there exists a constant $\rho \geq 0$ for any $n \in \RR$ such that
    \begin{align}
        \label{inequality:robustness-of-aggregation-local}
              & \|\mathcal{A}_n (\vx_n, \{\vx_{m,n}\}_{m\in \RR_n\cup \mathcal{B}_n} )-\hat \vx_n \| \\
         \leq & \rho \max_{m \in \RR_n \cup \{n\}}\|\vx_{m} - \hat \vx_n\|, \notag
    \end{align}
    then $W$ is the virtual mixing matrix and $\rho$ is the contraction constant associated with the robust aggregation rules $\{\mathcal{A}_n\}_{n\in \RR}$.
    % For the virtual mixing matrix $W$ associated with $\{\mathcal{A}_n\}_{n \in \RR}$, we also define $\beta := 1-\|(I-\frac{1}{\vert \RR \vert}\bm{1}\bm{1}^\top)W\|^2 \in (0,1]$.
\end{Definition}

The popular robust aggregation rules, such as TM, IOS and SCC, all satisfy Definition \ref{definition:mixing-matrix}.
Their virtual mixing matrices and contraction constants have been analyzed in \cite{Ye2023}. Note that the contraction constant is usually nonzero, and the virtual mixing matrix is not necessarily doubly stochastic, but only row sto- chastic. For later usage, we introduce $\chi^2 := \frac{1}{\vert \RR \vert}\|W^\top \bm{1}- \bm{1} \|^2$ to quantify the skewness of $W$, with $\bm{1}$ being a $|\mathcal{R}|$-dimensional all-one vector.
%\vspace{-0.5em}
%\begin{remark}
%    We investigate a general decentralized SGD framework equipped a class of Byzantine-resilient aggregation rules $\A_n$ satisfying Definition \ref{definition:mixing-matrix}. It has been proven in our previous work\cite{Ye2023} that many popular Byzantine-resilient aggregation rules fall into Definition \ref{definition:mixing-matrix}, such as trimmed mean, SCC, IOS.
%    % In addition, many aggregation rules that are successful in federated learning can be extended to the decentralized scenario and proven to satisfy Definition \ref{definition:mixing-matrix}, such as coordinate-wise median (CooMed) \cite{yin2018byzantine}, geometric median (GeoMed) \cite{chen2017distributed}, Krum \cite{blanchard2017machine}, to name a few.
%    The focus of this paper is not to study
%    $\rho$ and $\chi^2$ for different Byzantine-resilient aggregation rules. For a more comprehensive understanding of these aspects, we recommend readers to \cite{wu2022byzantine,Ye2023}.
%\end{remark}

\begin{algorithm}[t]
    \caption{Byzantine-Resilient DSGD}
    \label{robust-DSGD}
        \textbf{Input:} Initializations $\vx^0$ for all $n \in \RR$; step size $\alpha^k > 0$;
        \begin{algorithmic}[1]
        \ForAll {$k = 0, 1, 2, \cdots$}
        \ForAll {non-Byzantine agents $n \in \RR$}
        \State Compute stochastic gradient $\nabla f(\vx^{k}_n; \xi_n^{k})$
        \State Compute $\vx_n^{k+\frac{1}{2}} = \vx^k_n - \alpha^{k} \nabla f(\vx^{k}_n; \xi_n^{k})$
        \State Send $\vx_{n,m}^{k+\frac{1}{2}} = \vx_n^{k+\frac{1}{2}}$ to all neighbors $m \in \RR_n\cup \B_n$
        \State Receive $\vx_{m,n}^{k+\frac{1}{2}}$ from all neighbors $m \in \RR_n\cup \B_n$
        \State Aggregate $\vx^{k+1}_n= \mathcal{A}_n (\vx_n^{k+\frac{1}{2}}, \{ \vx_{m,n}^{k+\frac{1}{2}}\}_{m\in \RR_n\cup \B_n})$
        \EndFor
        \ForAll {Byzantine agents $n\in \B$}
        \State Send $\vx^{k+\frac{1}{2}}_{n,m}=*$ to all neighbors $m \in \RR_n\cup \B_n$
        \EndFor
        \EndFor
        \end{algorithmic}
\end{algorithm}

\section{Preliminaries}
\label{sec-pre}

This section outlines the assumptions made for the analysis, and introduces the technical tool of uniform stability.
\subsection{Assumptions}
We make several assumptions on the peer-to-peer communication network, the loss and the sampling process.

\begin{assumption}[Network connectivity]
    \label{assumption:connection}
    The subgraph comprised of all non-Byzantine agents $n \in \RR$ is connected. That is, there exists at least one undirected path between any pair $n, m\in \RR$. In addition, for for the virtual mixing matrix $W$, it holds that $\beta := 1-\|(I-\frac{1}{\vert \RR \vert}\bm{1}\bm{1}^\top)W\|^2 \in (0,1]$, with $\|\cdot\|$ being the matrix spectral norm.
\end{assumption}

Assuming network connectivity is necessary in analyzing Byzantine-resilient decentralized algorithms \cite{wu2022byzantine,fang2022bridge,he2022byzantine,Ye2023}. Without this assumption, in the presence of Byzantine attacks, non-Byzantine agents belonging to different subgraphs are only able to learn their ``isolated''  models, and reaching consensus is not guaranteed.

%With this assumption, for the virtual mixing matrix $W$ satisfying Definition \ref{definition:mixing-matrix}, we have
%$\beta := 1-\|(I-\frac{1}{\vert \RR \vert}\bm{1}\bm{1}^\top)W\|^2 \in (0,1]$.

% \begin{assumption}[Strong convexity]  \label{assumption:convex}
% 	The loss $f(\vx; \xi)$ is $\mu$-strongly convex.
%  %for each non-Byzantine agent $n\in \RR$.
% \end{assumption}

\begin{assumption}[Smoothness]  \label{assumption:Lip}
	%For each honest worker $n\in \RR$,
 The loss $f(\vx; \xi)$ is $L$-smooth.
\end{assumption}

\begin{assumption}[Bounded stochastic gradient]
	\label{assumption:gradients}
	The stochastic gradients $\nabla f(\vx_n^k;\xi_n^k)$ are upper-bounded by $M$ in the $\ell_2$-norm, over all times $k=0, 1, \cdots$ and across all non-Byzantine agents $n \in \RR$.
\end{assumption}

\begin{assumption}[Uniform and independent sampling]
	\label{assumption:indSampling}
    The trai- ning sample $\xi_n^{k}$ is uniformly and independently drawn from the local dataset $\cS_n$, over all times $k=0, 1, \cdots$ and across all non-Byzantine agents $n \in \RR$.	
%The stochastic gradients $\nabla f(\vx^{k}_n; \xi_n^{k})$ are drawn independently over all times $k=0, 1, \cdots$ and across all non-Byzantine agents $n \in \RR$. The sampling process is uniform and without replacement.
\end{assumption}

These assumptions are common in analyzing generalization errors of stochastic algorithms \cite{sun2021stability,deng2023stability,bars2023improved,hardt2016train}.
% We focus on the strongly convex case that satisfies Assumption \ref{assumption:convex}, and leave the general convex and non-convex cases to an extended version of this paper.
Note that Assumptions \ref{assumption:Lip} and \ref{assumption:gradients} can be relaxed. For example, they are relaxed to assuming H$\ddot{o}$lder continuous stochastic gradient in \cite{zhu2022topology}, thereby covering a class of functions positioned between those with smoothness and those with bounded stochastic gradients. We leave the investigation under relaxed assumptions as our future work.

%We conjecture that the conclusions drawn in this paper would hold true under relaxed assumptions, and leave it to future investigation.

\subsection{Uniform Stability}
Following the prior works \cite{sun2021stability,deng2023stability,bars2023improved}, we consider the generalization error of the average model at a given time $k$. In the context of Byzantine-resilient decentralized learning, the output $\mathcal{L}(\cS)$ of a stochastic algorithm $\mathcal{L}$ over a dataset $\cS$ at a given time $k$ is $\bar\vx^k := \frac{1}{|\RR|}\sum_{n \in \RR} \vx_n^k$, the average of all non-Byzantine agents' local models. The generalization error is defined for such $\mathcal{L}(\cS)$, with respect to discrepancy between the population risk in \eqref{1} and the global empirical risk in \eqref{erm}.

Our analysis of the generalization error relies on a popular tool named \textit{uniform stability} \cite{bousquet2002stability,hardt2016train}.

\begin{Definition}[Uniform stability \cite{bousquet2002stability}]
\label{def-sta}
A stochastic algorithm $\mathcal{L}$ is $\epsilon$-uniformly stable if for any two datasets $\cS$ and $\cS'$ of size $Z |\RR|$ that only differ in one training sample, we have
\begin{equation*}
\sup_{\xi} \E_{\mathcal{L}} \left[ f(\mathcal{L}(\cS); \xi) - f(\mathcal{L}(\cS'); \xi) \right] \le \epsilon,
\end{equation*}
where the training samples are drawn from $\cS$ and $\cS'$ in the same order and the expectation $\E_{\mathcal{L}}$ is taken over the random- ness of $\mathcal{L}$.
\end{Definition}

If a stochastic algorithm $\mathcal{L}$ is uniformly stable, then adding or removing a single training sample from the dataset does not matter too much to the trained model. This fact also implies that the trained model is not sensitive, and hence has a small generalization error. The connection between uniform  stability and generalization error is established in the following lemma.
 \begin{lemma}[Generalization error via uniform stability  \cite{hardt2016train}]
 \label{lemma-sta}
If a stochastic algorithm $\mathcal{L}$ is $\epsilon$-uniformly stable, then its generalization error satisfies
$
| \E_{\cS,\mathcal{L}}[F(\mathcal{L}(\cS))-F_{\cS}(\mathcal{L}(\cS))] |\leq \epsilon
$.
\end{lemma}

According to Lemma \ref{lemma-sta}, we can analyze the generalization error utilizing the tool of uniform stability,
% The relationship between algorithm stability and generalization error is rooted in the observation that when a replacement data sample in the training set substantially alters the algorithm's output, the algorithm may struggle to accurately predict this new data. This poor stability can result in a higher generalization error, as the algorithm's capacity to generalize from the training data to new samples is limited.
and the results are given in the ensuing section.

\section{Generalization Error Analysis}
\label{sec-ge}

Below, we analyze the generalization errors of Byzantine-resilient DSGD when the loss $f(\vx; \xi)$ is strongly convex, con- vex and non-convex, respectively. Therein, $1_{\chi \neq 0}$ outputs 1 if $\chi \neq 0$ while outputs 0 if $\chi=0$. All the proofs are deferred to the appendices.

\subsection{Generalization Error with Strongly Convex Loss}
\label{sec-ge-1}
\begin{theorem}[Generalization error of Byzantine-resilient DSGD with strongly convex loss]
\label{thm-ge}
Suppose that the robust aggregation rules $\{\A_n\}_{n\in \RR}$ in Algorithm \ref{robust-DSGD} satisfy Definition \ref{definition:mixing-matrix},  the associated contraction constant satisfies $\rho < \rho^* := \frac{\beta}{8\sqrt{\vert \RR \vert}}$ and the loss $f(\vx; \xi)$ is $\mu$-strongly convex. Set the step size $\alpha^k= \frac{1}{\mu (k+k_0)}$, where $k_0$ is sufficiently large.
Under Assumptions \ref{assumption:connection}--\ref{assumption:indSampling}, at any given time $k$, the generalization error of Algorithm \ref{robust-DSGD} is bounded by
    \begin{align}
        \label{thm-ge-1}
        & \E_{\cS,\mathcal{L}}[F(\bar \vx^k)-F_{\cS}(\bar \vx^k))] \leq \frac{c_2 1_{\chi \neq 0} M^2Lln(k+k_0)}{\mu^2(k+k_0-1)}  \\
        & +  \frac{2M^2}{\mu Z|\RR| } + \frac{4 c_1 \rho |\RR| M^2}{\mu} + \frac{2c_1\chi\sqrt{|\RR|} M^2}{\mu}. \nonumber
        % + \frac{c_2 (4\rho |\RR| +2\chi\sqrt{|\RR|}){\color{blue}H^0}M }{k+k_0-1}  \nonumber \\
        % &+ c_3 (4\rho |\RR| +2\chi\sqrt{|\RR|}) {\color{blue}H^0}M  (c_4)^k.  \nonumber
    \end{align}
     Here $c_1,c_2>0$ are
     % , $c_2 > 0$, $c_3 >0 $, and $c_4 \in (0,1)$ are
     constants.
\end{theorem}

When the loss is strongly convex, Theorem \ref{thm-ge} provides the generalization errors of a class of Byzantine-resilient DSGD algorithms equipped with robust aggregation rules that satisfy Definition \ref{definition:mixing-matrix}. Given that the Byzantine agents are not present and that the robust aggregation rules ensure proper mixing, Algorithm \ref{robust-DSGD} reduces to attack-free DSGD and the constants $\rho=\chi=0$. In this situation, the first, third and fourth terms at the right-hand side of \eqref{thm-ge-1} all become zero. Consequently, the generalization error is in the order of $O(\frac{M^2}{\mu Z|\RR| })$, which is independent on $k$. As the number of samples $Z$ and the number of agents $|\RR|$ increase to infinity, the generalization error tends to 0. Such a generalization error bound matches the ones obtained for SGD \cite{hardt2016train} and DSGD \cite{bars2023improved}, and is optimal when the loss is strongly convex \cite{zhang2022stability}. It is notably tighter than those established in \cite{sun2021stability} and \cite{deng2023stability}, in which the bounds contain additional terms that do not vanish when $Z$ and $|\RR|$ go to infinity.

% \noindent\textbf{Impact of Byzantine agents.}
When the Byzantine agents are present, generally speaking the contraction constant $\rho \neq 0$. In this situation, the third term at the right-hand side of \eqref{thm-ge-1} is in the order of $O(\frac{ \rho |\RR| M^2}{\mu})$. Furthermore, when the virtual mixing matrix $W$ is non-doubly stochastic such that $\chi \neq 0$, we have two additional terms: the first term vanishes when $k$ goes to infinity, but the fourth term does not vanish and is in the order $O(\frac{ \chi \sqrt{|\RR|} M^2}{\mu})$.

\subsection{Generalization Error with Convex Loss}
\label{sec-ge-2}
\begin{theorem}[Generalization error of Byzantine-resilient DSGD with convex loss]
\label{thm-ge-c}
Suppose that the robust aggregation rules $\{\A_n\}_{n\in \RR}$ in Algorithm \ref{robust-DSGD} satisfy Definition \ref{definition:mixing-matrix}, the associated contraction constant satisfies $\rho < \rho^* := \frac{\beta}{8\sqrt{\vert \RR \vert}}$ and the loss $f(\vx; \xi)$ is convex. Set the step size $\alpha^k= \frac{1}{k+k_0}$, where $k_0$ is sufficiently large.
Under Assumptions \ref{assumption:connection}--\ref{assumption:indSampling}, at any given time $k$, the generalization error of Algorithm \ref{robust-DSGD} is bounded by
    \begin{align}
        \label{thm-ge-c-1}
        & \E_{\cS,\mathcal{L}}[F(\bar \vx^k)-F_{\cS}(\bar \vx^k)]
        \leq    c_2 1_{\chi \neq 0} M^2 L  \\ & + \lp \frac{2M^2}{ Z|\RR| } + 4 c_1 \rho |\RR| M^2 +2 c_1 \chi\sqrt{|\RR|} M^2 \rp ln(k+k_0).  \nonumber
        % &+ c_6 (4\rho |\RR| +2\chi\sqrt{|\RR|})H^0 M  .  \nonumber
       % & + \frac{c_6 (4\rho |\RR| +2\chi\sqrt{|\RR|})H^0Mk_0}{k+k_0-1}  \nonumber
    \end{align}
     Here $c_1, c_2>0$ are constants.
\end{theorem}

Similar to the discussion in Section \ref{sec-ge-1}, when $\rho=\chi=0$, the derived generalization error bound in Theorem \ref{thm-ge-c} is in the order of $O(\frac{M^2 lnk}{ Z|\RR| })$, recovers the ones obtained for SGD \cite{hardt2016train} and DSGD \cite{bars2023improved}, and is optimal when the loss is convex \cite{zhang2022stability}.
% It also matches the lower bound for convex loss established in \cite{zhang2022stability}.

% \noindent\textbf{Impact of iteration number.} Differing from the time-independent generalization error bound in Theorem \ref{thm-ge}, we derive a time-increasing generalization error bound in Theorem \ref{thm-ge-c}. This bound aligns with the lower bound of centralized SGD for convex loss as established in \cite{zhang2022stability}. The time-increasing generalization error is also empirically observed in experiments in Section \ref{sec-num}, validating our theoretical analysis.
% Theorem \ref{thm-ge-c} suggests that the smaller the iteration number $k$, the better the generalization performance.
% Thus, Theorem \ref{thm-ge-c} reflects the  ``train faster, generalize better'' principle demonstrated in  \cite{hardt2016train}.
% As a practical approach, stopping the training process early upon reaching a low training error could significantly improve generalization performance.
% % Moreover, a comparison between Theorem \ref{thm-ge} and Theorem \ref{thm-ge-c} reveals that integrating a strongly-convex regularizer often leads to improved generalization performance, corroborating findings from previous studies \cite{hardt2016train,shalev2010learnability,zhang2022stability}.

% \noindent\textbf{Impact of Byzantine agents.}
When the Byzantine agents are present, the generally non-zero contraction constant $\rho$ introduces a generalization error term in the order of $O( \rho |\RR| M^2 lnk)$. In addition, a non-doubly stochastic virtual mixing matrix $W$, which implies $\chi \neq 0$, yields a constant term in the order of $O(M^2L)$ and a time-dependent term $O( \chi \sqrt{|\RR|} M^2 lnk)$.

\begin{remark}
    Comparing Theorems \ref{thm-ge} and \ref{thm-ge-c}, we observe that the generalization error with the convex loss sublinearly increases with $k$, while that with the strongly convex loss does not. This comparison highlights that incorporating a strongly convex regularization term into the convex loss leads to a significantly improved generalization error bound. Such an observation aligns with the conclusions drawn in the prior works \cite{hardt2016train,shalev2010learnability,zhang2022stability} and will be validated in the numerical experiments in Section \ref{sec-num}.
\end{remark}
% \noindent\textbf{Impact of regularization term.}
%

\begin{figure*}[h]
\centering
{\includegraphics[width=7in]{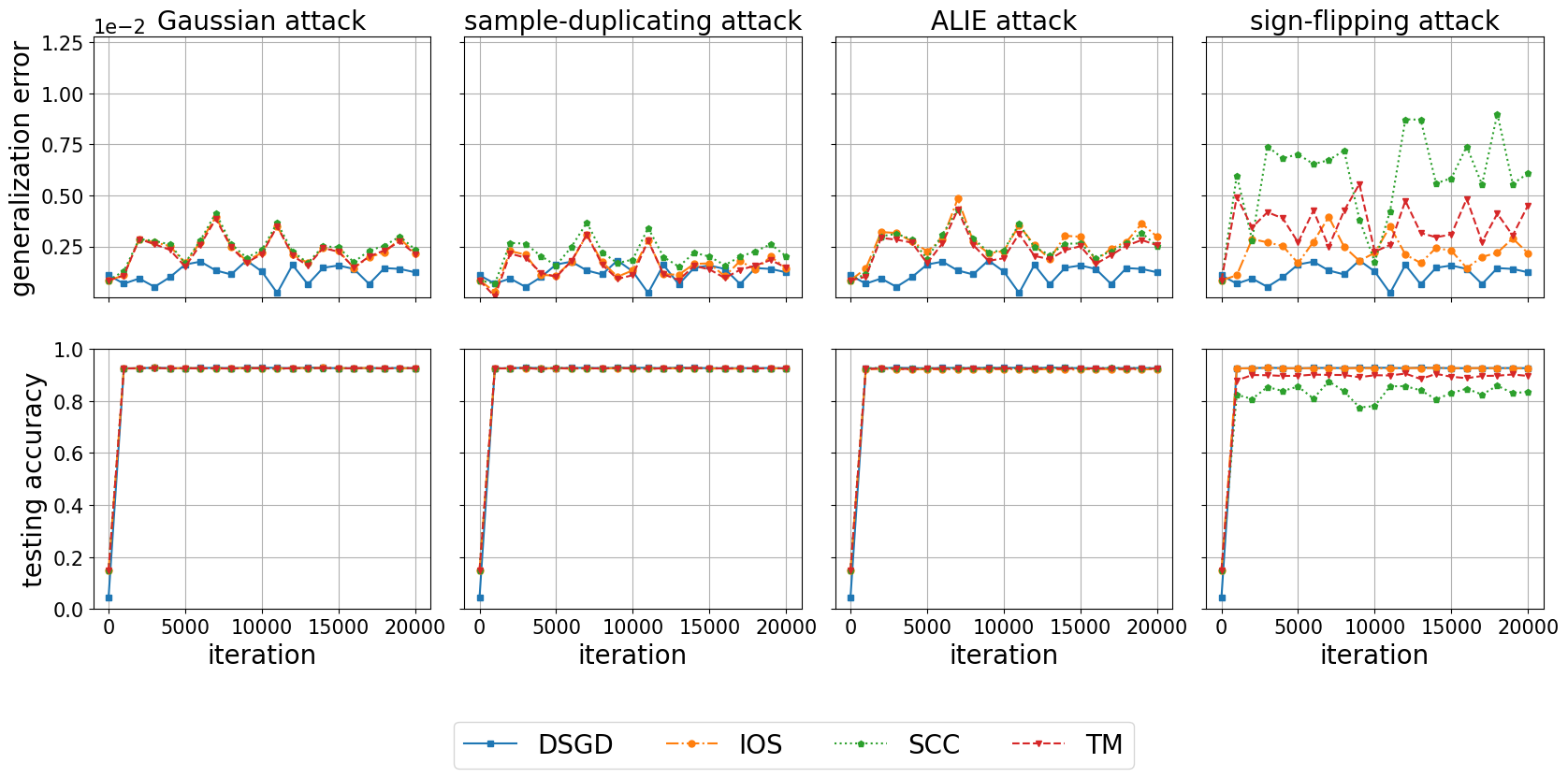}}
\caption{Generalization error and testing accuracy of attack-free and Byzantine-resilient DSGD with strongly convex loss.}\label{sc-iid}
\end{figure*}

\begin{figure*}[h]
\centering
{\includegraphics[width=7in]{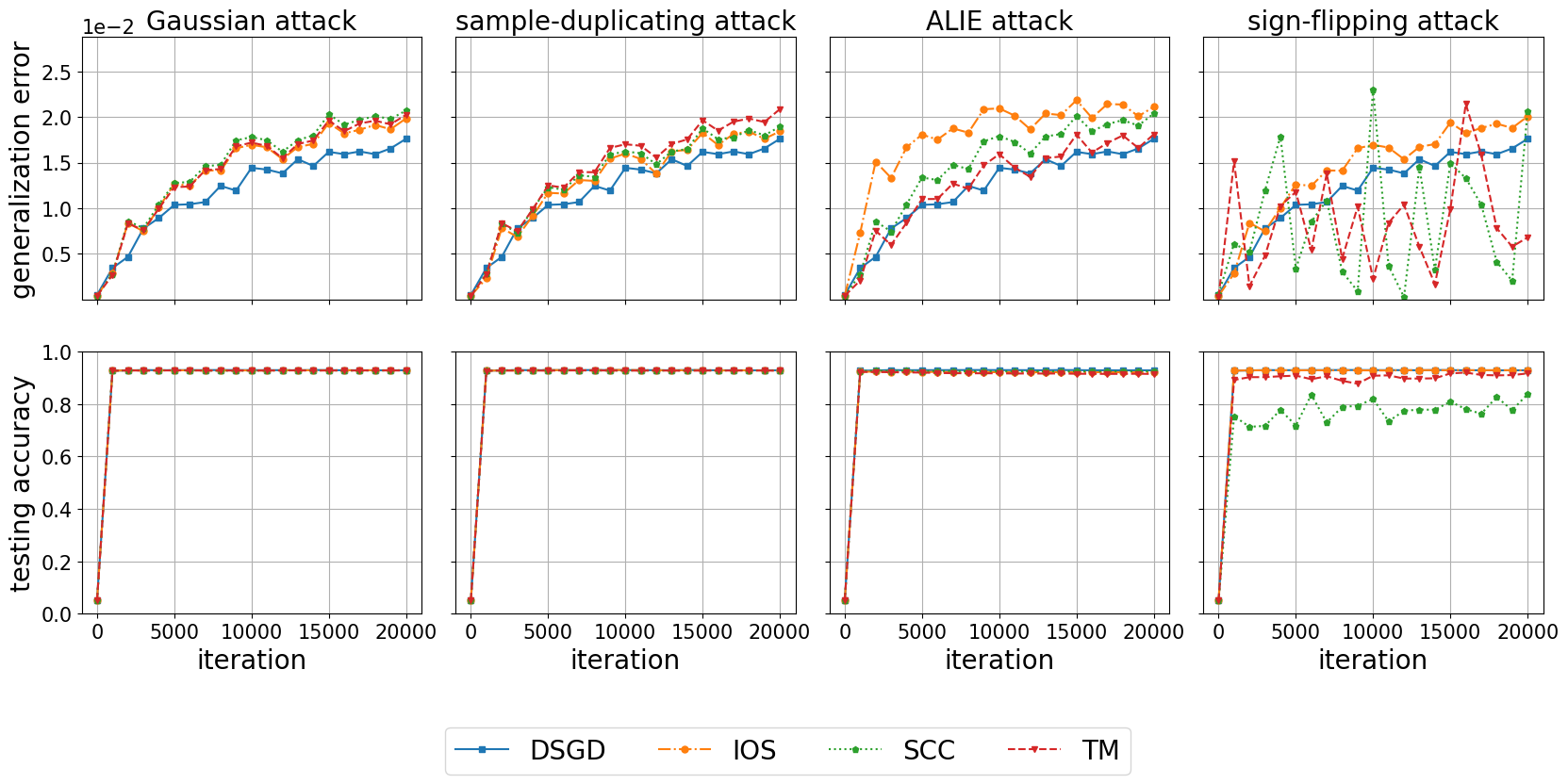}}
\caption{Generalization error and testing accuracy of attack-free and Byzantine-resilient DSGD with convex loss.}\label{c-iid}
\end{figure*}

\begin{figure*}[h]
\centering
{\includegraphics[width=7in]{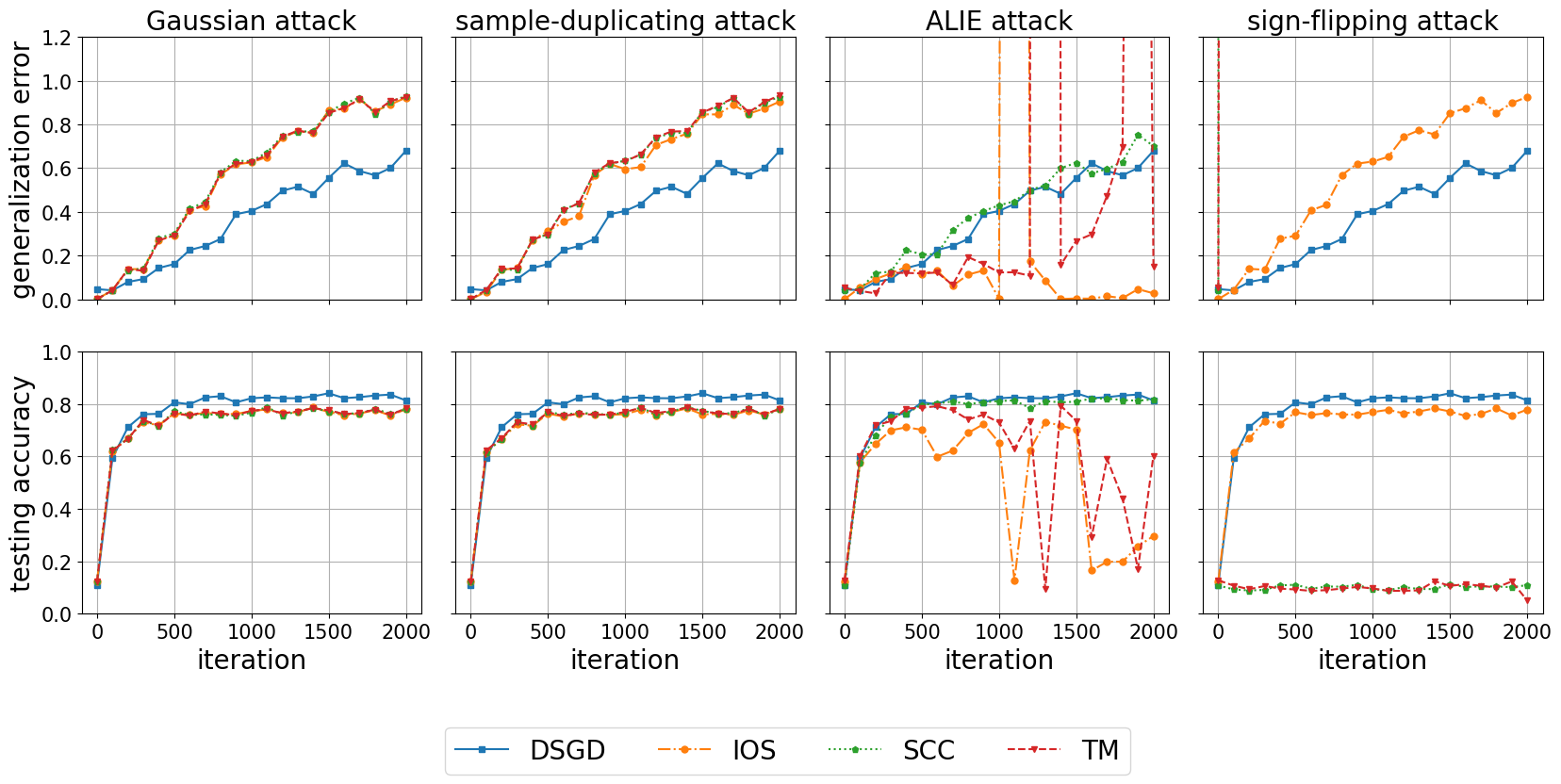}}
\caption{Generalization error and testing accuracy of attack-free and Byzantine-resilient DSGD with non-convex loss.}\label{nc-iid}
\end{figure*}

\begin{figure*}[h]
\centering
{\includegraphics[width=7in]{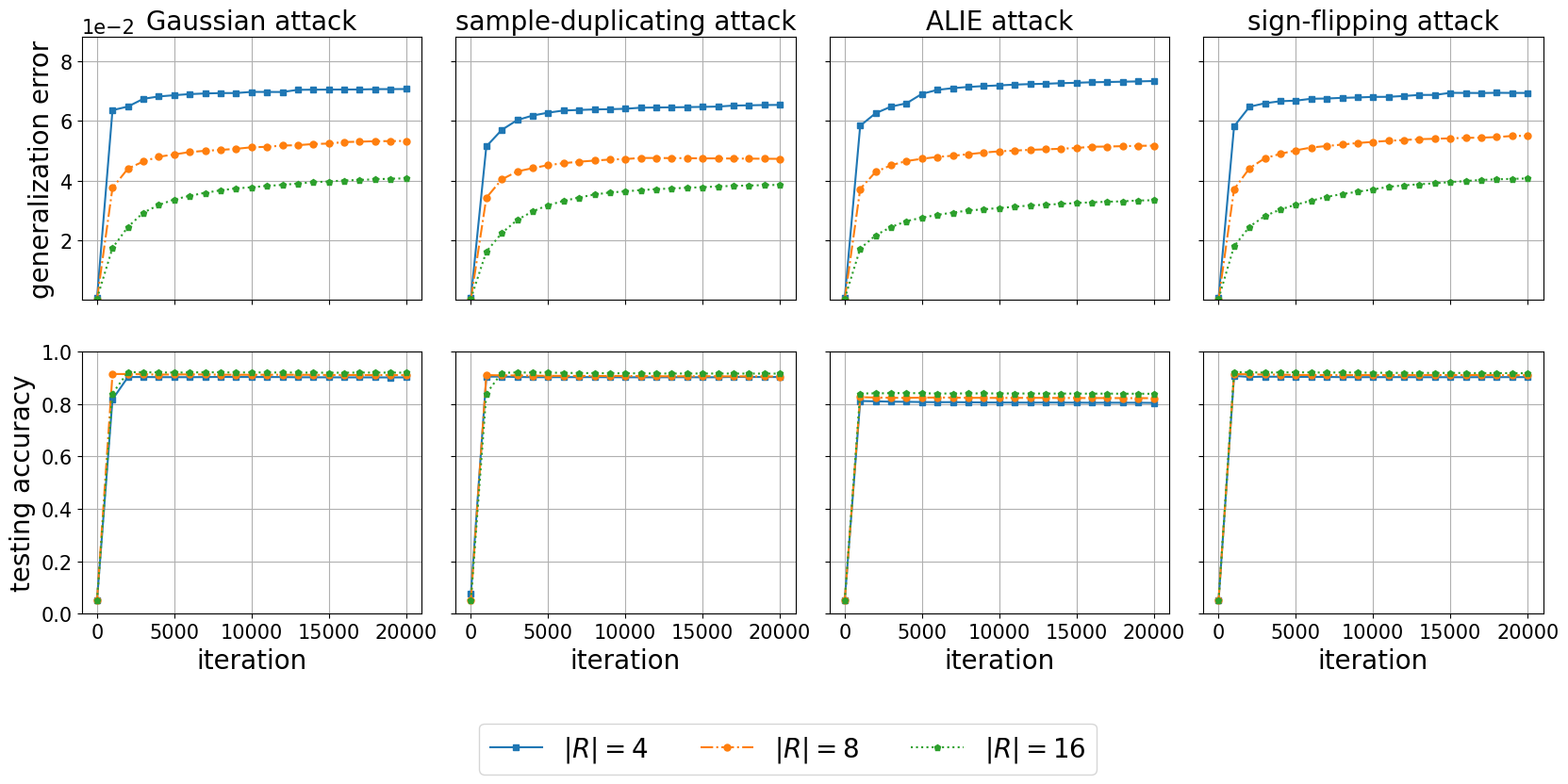}}
\caption{Generalization error and testing accuracy of Byzantine-resilient DSGD using IOS with convex loss and different $|\RR|$.}\label{ios-r}
\end{figure*}

\subsection{Generalization Error with Non-Convex Loss}
\label{sec-ge-3}
\begin{theorem}[Generalization error of Byzantine-resilient DSGD with non-convex loss]
\label{thm-ge-n}
Suppose that the robust aggregation rules $\{\A_n\}_{n\in \RR}$ in Algorithm \ref{robust-DSGD} satisfy Definition \ref{definition:mixing-matrix}, the associated contraction constant satisfies $\rho < \rho^* := \frac{\beta}{8\sqrt{\vert \RR \vert}}$, and the loss $f(\vx; \xi)$ is non-convex. Set the step size $\alpha^k= \frac{1}{L (k+k_0)}$, where $k_0$ is sufficiently large.
Under Assumptions \ref{assumption:connection}--\ref{assumption:indSampling}, at any given time $k$, the generalization error of Algorithm \ref{robust-DSGD} is bound- ed by
    \begin{align}
        \label{thm-ge-n-1}
        & \E_{\cS,\mathcal{L}}[F(\bar \vx^k)-F_{\cS}(\bar \vx^k)] \leq  \frac{c_2 1_{\chi \neq 0} M^2  (k+k_0)}{L} \\
         &  + \lp \frac{2M^2}{L Z|\RR| } +  \frac{ 4c_1\rho |\RR| M^2}{L} +  \frac{ 2c_1\chi\sqrt{|\RR|} M^2}{L} \rp   (k+k_0). \nonumber
        % &+ c_6 (4\rho |\RR| +2\chi\sqrt{|\RR|})H^0 M  .  \nonumber
       % & + \frac{c_6 (4\rho |\RR| +2\chi\sqrt{|\RR|})H^0Mk_0}{k+k_0-1}  \nonumber
    \end{align}
     Here $c_1, c_2>0$ are constants.
\end{theorem}

Again, in the case that $\rho=\chi=0$, the derived generalization error bound in Theorem \ref{thm-ge-n} is in the order of $O(\frac{M^2 k}{L Z|\RR| })$.

When the Byzantine agents are present such that the contraction constant $\rho \neq 0$ in general, the induced generalization error term is in the order of $O(\frac{ \rho |\RR| M^2 k}{L})$. Furthermore, a non-doubly stochastic virtual mixing matrix $W$, which implies $\chi$ $\neq 0$, yields two time-dependent terms in the orders of $O(\frac{M^2k}{L})$ and $O(\frac{ \chi \sqrt{|\RR|} M^2 k}{L})$, respectively. 

Note that the generalization error bound  in Theorem \ref{thm-ge-n} can be improved by considering the first time when the different training samples from $\cS$ and $\cS'$ are selected. We leave it to Theorem \ref{thm-ge-in} in the Appendix \ref{app-im}.

% \noindent\textbf{Impact of communication network.} When there are no Byzantine agents, the generalization error is independent of network topology.  However, in the presence of Byzantine agents,  a sparse topology will increase $\rho$ and magnify the additional error term of $O(\frac{(\rho  +\chi) M^2 k}{L})$ in generalization error.

% \noindent\textbf{Impact of iteration number.} In Theorem \ref{thm-ge-n}, we derive a time-increasing generalization error bound. The time-increasing generalization error is also empirically observed in experiments in Section \ref{sec-num}, validating our theoretical analysis.
%  Theorem \ref{thm-ge-n} also reflects the  ``train faster, generalize better'' principle demonstrated in  \cite{hardt2016train}.
% As a practical approach, stopping the training process early upon reaching a low training error could significantly improve generalization performance.
\subsection{New Features}

\begin{table*}[!htbp]
       \normalsize
    \caption{Generalization errors of attack-free DSGD and Byzantine-resilient DSGD with different losses.}
    \label{table-result}
    \centering
    \begin{tabular}{ccc}
    \hline
    loss     &attack-free DSGD  & Byzantine-resilient DSGD  \\
    \hline
    strongly convex & $O\lp \frac{M^2}{ \mu Z|\RR| } \rp$ & $O\lp \frac{M^2}{\mu Z|\RR| } \rp  + O\lp \frac{ \rho |\RR| M^2}{\mu} \rp + O \lp \frac{ \chi \sqrt{|\RR|} M^2}{\mu} \rp + O\lp \frac{ 1_{\chi \neq 0} M^2 L ln k}{\mu^2k} \rp$  \\
    convex  & $O\lp \frac{M^2 ln k}{Z|\RR| } \rp$ &   $O \lp \frac{ M^2 ln k}{Z|\RR| } \rp + O\lp \rho |\RR| M^2 lnk\rp + O\lp \chi \sqrt{|\RR|} M^2 lnk\rp + O \lp 1_{\chi \neq 0} M^2 L \rp $ \\
    non-convex   & $O\lp \frac{M^2 k}{LZ|\RR| } \rp$ & $O \lp \frac{ M^2 k}{LZ|\RR| } \rp + O\lp \frac{\rho |\RR| M^2 k}{L} \rp + O\lp \frac{ \chi\sqrt{|\RR|} M^2 k}{L} \rp + O\lp \frac{1_{\chi \neq 0} M^2 k}{L} \rp$ \\
    \hline
    \end{tabular}
\end{table*}

We summarize the established generalization error bounds of Byzantine-resilient DSGD in Table \ref{table-result}. Compared with those of attack-free DSGD, there are some new features.
\begin{itemize}
    % \item \textbf{Time Dependence:} The generalization error is dependent on time $k$. When we train for a longer time, the generalization error tends to improve.
     \item \textbf{Additional Error Terms:} The presence of the Byzantine agents introduces additional error terms.
     They appear to be independent on $Z$, the number of training samples per agent. Instead, they are related to the communication topology and the robust aggregation rules (characterized by $\rho$ and $\chi$), the number of non-Byzantine agents $|\mathcal{R}|$, as well as the loss (characterized by $L$ and $M$).
    \item \textbf{Impossibility of Converging to Zero:} The generalization errors are unable to converge to zero when the number of training samples per agent $Z$ and/or the number of non-Byzantine agents $|\mathcal{R}|$ increase to infinity, because of the additional error terms.
    \item \textbf{Communication Topology Dependence:} For attack-free DSGD, the generalization errors are independent on the communication topology.
    % , aligning the conclusion in \cite{bars2023improved,richards2020graph}.
    % The generalization error obtained in \eqref{thm-ge-1} confirms the conclusion that generalization error is completely independent of the network topology as stated in \cite{bars2023improved,richards2020graph}.
    However, in the presence of the Byzantine agents, the communication topology plays an important role.
The underlying reason is that the negative impact of the Byzantine agents is closely linked to the communication topology.
For a fixed robust aggregation rule, different communication topologies lead to different $\rho$ and $\chi$, thus leading to different additional error terms.
% When the number of Byzantine agents is fixed, a sparse topology will increase $\rho$ and magnify the additional error term in generalization error.
\end{itemize}

These features highlight the negative impact of the Byzantine attacks on the generalization error of Byzantine-resilient decentralized learning. To understand the insight, imagine that the arbitrarily malicious messages sent by the Byzantine agents can be viewed as from arbitrarily corrupted training samples within the dataset $\cS$. Because we are unable to completely remove these arbitrarily corrupted training samples in practice, the contaminated dataset $\cS$ is unable to represent the true underlying data distribution. In consequence, the generalization error can never be zero.

\begin{remark}
The presence of the Byzantine agents introduces new challenges to establishing the generalization error, as the arbitrarily malicious messages, the robust aggregation rules,
and the possibly row stochastic virtual mixing matrices comp- licate the stability analysis. Traditional technical tools that are effective in analyzing attack-free DSGD with doubly stochastic mixing matrices, such as those in Lemma 8 of \cite{sun2021stability} and Lemma 5 of \cite{deng2023stability}, are no longer applicable.

To address these challenges, we decompose the difference of the algorithm outputs on datasets $\cS$ and $\cS'$, which is denoted by $\|\bar\vx^{k+1} - \bar\vx'^{k+1} \|$, into $\|\bar\vx^{k+1} - \bar\vx^{k+\frac{1}{2}} \| +  \|\bar\vx^{k+\frac{1}{2}} - \bar\vx'^{k+\frac{1}{2}} \| + \|\bar\vx'^{k+1} - \bar\vx'^{k+\frac{1}{2}} \|$. To handle the first and the third terms, we leverage Definition \ref{definition:mixing-matrix} and Assumption \ref{assumption:connection} to bound the incre- ments. To handle the second term, we consider two scenarios, depending on whether different training samples are selected at time $k$. We then bound the impact of the possibly different training samples using Assumptions \ref{assumption:gradients}--\ref{assumption:indSampling}. Finally, leveraging Assumption \ref{assumption:Lip} and setting appropriate step sizes, we establish the recursion between $\|\bar\vx^{k+1} - \bar\vx'^{k+1} \|$ and $\|\bar\vx^{k} - \bar\vx'^{k} \|$, which enables us to prove the theorems by induction.
\end{remark}

\begin{remark}
In this paper, we only focus on the generalization error of Byzantine-resilient decentralized learning. Our prior work \cite{Ye2023} has investigated the optimization error, denoted as $\E_{\cS,\mathcal{L}}[F_{\cS}(\bar\vx^k)-F_{\cS}(\vx^*_{\cS})]$. When the loss is $\mu$-strongly convex but not $L$-smooth, while the stochastic gradients are bounded by $M$, the optimization error of Algorithm \ref{robust-DSGD} is in the order of $O ( \frac{\rho^2 |\RR| M^2}{\mu} )  + O ( \frac{\chi^2 M^2}{\mu} )$ given that the training time $k$ goes to infinity. We can conclude that the presence of the Byzantine agents introduces non-vanishing terms for both optimization and generalization errors, even when the training time $k$ and the number of training samples per agent $Z$ are infinite.

Solely minimizing either the optimization error or the generalization error is insufficient to guarantee a favorable testing accuracy. For example, a random guess exhibits the best gene- ralization ability since it behaves consistently and randomly on both the training and testing samples. Nevertheless, its optimization error is very large, which eventually leads to an unsatisfactory testing accuracy. We will investigate the inter- play between the generalization and optimization errors, as well as their joint minimization in our future work.
\end{remark}

\begin{remark}

According to Theorems \ref{thm-ge}--\ref{thm-ge-n}, when the loss function is not strongly convex, the established generalization error bounds increase with time, eventually tending towards infinity as $k \rightarrow \infty$. Although such time-increasing generalization error bounds are common in the analyses of SGD and DSGD \cite{sun2021stability,deng2023stability,zhu2022topology,bars2023improved,richards2020graph,taheri2023generalization,hardt2016train,zhang2022stability},
    they become unsatisfactory for a large $k$, as experimental evidences typically show that the generalization abilities stabilize rather than deteriorate indefinitely.

We assert that such generalization error bounds are meaningful when $k$ is small.
    Actually, according to \cite{chen2018stability}, for any iterative algorithm, the expected excess risk is lower-bounded by a constant, e.g., $O(\frac{1}{\sqrt{Z|\RR|}})$ when the loss is convex.
     This implies that the generalization error should increase with the time in the beginning, due to the decreasing optimization error. We also observe the time-increasing generalization errors in our numerical experiments, when the losses are convex and non-convex (see Section \ref{sec-num}).
     %This demonstrates that our results, as well as those in \cite{sun2021stability,deng2023stability,zhu2022topology,bars2023improved,richards2020graph,taheri2023generalization,hardt2016train,zhang2022stability}, are meaningful across a broad range of time $k$.

    In addition, the time-increasing generalization error bounds are potentially able to be improved to time-uniform ones with further assumptions and advanced theoretical tools. A number of recent works have established time-uniform generalization error bounds \cite{farghly2021time,zhu2024uniform,hendrickx2024convex}.
    Among them, \cite{farghly2021time} and \cite{zhu2024uniform} respectively obtain time-uniform bounds for
    stochastic gradient Lan- gevin dynamics (SGLD) and SGD when the loss is non-convex, under an additional dissipativity assumption on the loss and using the technical tool of Wasserstein distance. The work of \cite{hendrickx2024convex} establishes a time-uniform generalization error bound for projected SGD when the loss is convex, assuming that the po- pulation risk is defined on a compact convex set. Unlike the technical tool of uniform stability that we use, \cite{hendrickx2024convex} investigates the algorithm outputs of two datasets differing by all training

    \noindent samples other than only one training sample.
    Although these bounds have their disadvantages, such as poor dependence on the model dimension or the number of training samples,
    they provide new insights into the evolution of the generalization ability.
%    However, the bound obtained in \cite{farghly2021time} relies on the additional Gaussian noise added to the gradients, and the bound obtained in \cite{zhu2024uniform} does not vanish even when the sample size $n \rightarrow \infty$.
    In this paper, we make the first step towards understanding the generalization ability in
    decentralized learning under Byzantine attacks. Deriving time-uniform generalization error bounds with the help of novel technical tools will be one of our future research directions.
%    Developing proper tools to analyze the generalization error is challenging, and
%    we will try to derive a time-uniform generalization error bound in our future work.
\end{remark}

\section{Numerical Experiments}
\label{sec-num}

In the numerical experiments, we construct an Erdos-R{\'e}nyi graph of $10$ agents, and let each pair of agents be connected with probability 0.7. By default, we randomly select $2$ agents to be Byzantine. We consider three tasks: squared $\ell_2$-norm regularized softmax regression on the MNIST dataset, softmax regression on the MNIST dataset and ResNet-18 training on the CIFAR-10 dataset. They correspond to the strongly convex, convex and non-convex losses, respectively. In each task, the training samples are evenly and uniform randomly allocated to all the non-Byzantine agents. During training, the batch size is set to 256. The step size is set to $\alpha^k = \frac{1}{0.01k+1}$ for the strongly convex and convex losses, and $\alpha^k = \frac{0.015}{0.01k+1}$ for the non-convex loss.

We implement Byzantine-resilient DSGD outlined in Algorithm \ref{robust-DSGD} equipped with different robust aggregation rules: IOS \cite{wu2022byzantine}, SCC \cite{he2022byzantine} and TM \cite{fang2022bridge}. The attack-free DGSD without Byzantine agents is used as the baseline. We evaluate the per- formance with two metrics: generalization error and testing accuracy of the non-Byzantine agents' average model $\bar{\vx}^k$. The generalization error is approximated by the difference between the losses on randomly selected testing and training samples, following \cite{deng2023stability,zhu2022topology}. The code is available online\footnote{\url{https://github.com/haoxiangye/BRDSGD-GE}}.

We consider the following Byzantine attacks.

\noindent\textbf{Gaussian Attack\cite{dong2024}.} The Byzantine agents send messages whose elements follow a Gaussian distribution with mean $0$ and variance $900$. \\
\noindent\textbf{Sample-Duplicating Attack\cite{rsa}.} The Byzantine agents col- laboratively select one of the non-Byzantine agents and always duplicate its messages to send. This is equivalent to that the Byzantine agents duplicate the training samples of the selected non-Byzantine agent. \\
\noindent\textbf{A-Little-Is-Enough (ALIE) Attack\cite{baruch2019little}.} To non-Byzantine agent $n$, its Byzantine neighbors send $\frac{1}{|\RR_n|}\sum_{m\in \RR_n }\vx_{m,m}^{k+\frac{1}{2}} + r_n^k \Delta_n^k$ at the $k$-th time, in which $\Delta_n^k$ is the coordinate-wise standard deviation of $\{\vx_{m,m}^{k+\frac{1}{2}}\}_{m\in \RR_n}$ and $r_n^k$ is the scale factor. \\
\noindent\textbf{Sign-Flipping Attack\cite{wu2022byzantine}.} To non-Byzantine agent $n$, its Byzantine neighbors multiply $\frac{1}{|\RR_n|}\sum_{m\in \RR_n }\vx_{m,m}^{k+\frac{1}{2}}$ with a negative constant $-1$ and send the result, at the $k$-th time.
% \noindent\textbf{Isolating Attacks.} The Byzantine agents send crafted messages in a way that their non-Byzantine neighbors's local models remain the same after weighted averaging in DSGD. This is equivalent to isolating their non-Byzantine neighbors.

\subsection{Strongly Convex Loss}
\label{sec-num-sc}
The generalization errors and testing accuracies of attack-free and Byzantine-resilient DSGD algorithms with the strongly convex loss are illustrated in Fig. \ref{sc-iid}.
We can observe that the generalization errors fluctuate but do not increase when $k$ increases. We can also observe that the generalization errors of Byzantine-resilient DSGD are larger than that of attack-free DGSD, which validates the additional error terms revealed by Theorem \ref{thm-ge}. Among the robust aggregation rules, IOS turns to be better than TM and SCC. Fig. \ref{sc-iid} also demonstrates that a smaller generalization error is helpful to reach a higher testing accuracy, though the latter involves the optimization error that is not the focus of this paper.
%This observation underscores the significance of investigating the generalization error for Byzantine-resilient decentralized learning.

\subsection{Convex Loss}
\label{ex-c}
The generalization errors and testing accuracies of attack-free and Byzantine-resilient DSGD algorithms with the convex loss are illustrated in Fig. \ref{c-iid}. The generalization errors increase as the time $k$ evolves, matching the findings of Theorem \ref{thm-ge-c}. This phenomenon demonstrates that a proper strongly convex regularization term can aid generalization, which is consistent with the conclusions made in \cite{hardt2016train,shalev2010learnability,zhang2022stability}. Observe that the generalization errors of Byzantine-resilient DSGD increase faster than that of attack-free DGSD due to the additional error terms; see Theorem \ref{thm-ge-c}. Under the sign-flipping attack, the ge-
neralization errors of SCC and TM are unstable, and the testing accuracies are also degraded.

\subsection{Non-Convex Loss}
The generalization errors and testing accuracies of attack-free and Byzantine-resilient DSGD algorithms with the non-convex loss are illustrated in Fig. \ref{nc-iid}. Similar to the observations for the convex loss, the generalization errors increase when $k$ increases and Byzantine-resilient DSGD has larger generalization errors than attack-free DSGD as predicted by Theorem \ref{thm-ge-n}. Under the ALIE attack, IOS and TM exhibit unstable generalization errors, accompanied by unstable testing accuracies.
The sign-flipping attack turns out to be the strongest, causing TM and SCC to have very large generalization errors and low testing accuracies from the beginning.

\subsection{Impact of $|\RR|$}
\label{ex-co}
In Theorems \ref{thm-ge}--\ref{thm-ge-n}, the number of the non-Byzantine agents $|\RR|$ appears in the denominators of some terms of the generalization errors, and in the numerators of some other terms. Therefore, a natural question arises: what is the impact of the number of the non-Byzantine agents?
To answer this question, we consider a special case where communication topology is fully connected, the aggregation rule is IOS and the loss is convex.
As shown in \cite{Ye2023}, we have $\rho = \frac{|\B|}{|\RR|}$ and $\chi=0$ in this case. Consequently, the terms of $O (\chi \sqrt{|\RR|} M^2 lnk )$ and $O ( 1_{\chi \neq 0} M^2 L )$  in \eqref{thm-ge-c-1} both vanish. Because $\rho = \frac{|\B|}{|\RR|}$, the term of $O\lp \rho |\RR| M^2 lnk\rp$ becomes $O( |\B| M^2 lnk )$. When $|\B|$ is fixed, the term of $O ( \frac{M^2 ln k}{Z|\RR| } )$ reveals that a large number of the non-Byzantine agents is beneficial to the generalization ability.

We conduct numerical experiments to validate this finding.
We fix $|\B| = 2$ and $Z = 1000$, and test Algorithm \ref{robust-DSGD} equipped with IOS in a fully connected communication topology with varying $|\RR|$. Other parameter settings are consistent with those in Section \ref{ex-c}.
 The generalization errors and testing accuracies are in Fig. \ref{ios-r}. Observe that the generalization error decreases as $|\RR|$ increases, confirming  our theoretical analysis.

\section{Conclusions}
\label{sec-con}
In this paper, we present the first generalization error analysis for a class of Byzantine-resilient decentralized learning algorithms. Our theoretical results reveal that the generalization error cannot vanish even when the number of training samples and the number of non-Byzantine agents go to infinity. These theoretical results are validated by the numerical experiments. We hope that our preliminary work is able to attract more researches in investigating the generalization error for Byzantine-resilient decentralized learning.
%In addition, we introduce novel modifications to Byzantine-resilient decentralized learning algorithms aimed at enhancing the generalization error.
% Numerical experiments are conducted to validate theoretical results.

% \noindent \textbf{Acknowledgement.} The work of Qing Ling (corresponding author) is supported in part by NSF China grants 61973324, 12126610 and 62373388, Guangdong Basic and Applied Basic Research Foundation grant 2021B1515020094, and Guangdong Provincial Key Laboratory of Computational Science grant 2020B1212060032.

\bibliographystyle{IEEEtran}
\bibliography{TSP}

\begin{thebibliography}{10}
\providecommand{\url}[1]{#1}
\csname url@samestyle\endcsname
\providecommand{\newblock}{\relax}
\providecommand{\bibinfo}[2]{#2}
\providecommand{\BIBentrySTDinterwordspacing}{\spaceskip=0pt\relax}
\providecommand{\BIBentryALTinterwordstretchfactor}{4}
\providecommand{\BIBentryALTinterwordspacing}{\spaceskip=\fontdimen2\font plus
\BIBentryALTinterwordstretchfactor\fontdimen3\font minus \fontdimen4\font\relax}
\providecommand{\BIBforeignlanguage}[2]{{%
\expandafter\ifx\csname l@#1\endcsname\relax
\typeout{** WARNING: IEEEtran.bst: No hyphenation pattern has been}%
\typeout{** loaded for the language `#1'. Using the pattern for}%
\typeout{** the default language instead.}%
\else
\language=\csname l@#1\endcsname
\fi
#2}}
\providecommand{\BIBdecl}{\relax}
\BIBdecl

\bibitem{ye2024ge}
H.~Ye and Q.~Ling, ``On the generalization error of {B}yzantine-resilient decentralized learning,'' \emph{International Conference on Acoustics, Speech and Signal Processing}, 2024.

\bibitem{McMahan2016}
B.~McMahan, E.~Moore, D.~Ramage, S.~Hampson, and B.~A. y~Arcas, ``Communication-efficient learning of deep networks from decentralized data,'' \emph{International Conference on Artificial Intelligence and Statistics}, 2017.

\bibitem{10180365}
H.~Zhang, K.~Zeng, and S.~Lin, ``Fed{UR}: Federated learning optimization through adaptive centralized learning optimizers,'' \emph{IEEE Transactions on Signal Processing}, vol.~71, pp. 2622--2637, 2023.

\bibitem{Bian2024}
J.~Bian, L.~Wang, K.~Yang, C.~Shen, and J.~Xu, ``Accelerating hybrid federated learning convergence under partial participation,'' \emph{IEEE Transactions on Signal Processing}, 2024.

\bibitem{Lian2017}
X.~Lian, C.~Zhang, H.~Zhang, C.-J. Hsieh, W.~Zhang, and J.~Liu, ``Can decentralized algorithms outperform centralized algorithms? {A} case study for decentralized parallel stochastic gradient descent,'' \emph{Advances in Neural Information Processing Systems}, 2017.

\bibitem{9713700}
W.~Liu, L.~Chen, and W.~Zhang, ``Decentralized federated learning: Balancing communication and computing costs,'' \emph{IEEE Transactions on Signal and Information Processing over Networks}, vol.~8, pp. 131--143, 2022.

\bibitem{9802673}
S.~A. Alghunaim and K.~Yuan, ``A unified and refined convergence analysis for non-convex decentralized learning,'' \emph{IEEE Transactions on Signal Processing}, vol.~70, pp. 3264--3279, 2022.

\bibitem{chen2024}
L.~Chen, W.~Liu, Y.~Chen, and W.~Wang, ``Communication-efficient design for quantized decentralized federated learning,'' \emph{IEEE Transactions on Signal Processing}, vol.~72, pp. 1175--1188, 2024.

\bibitem{nedic2009distributed}
A.~Nedic and A.~Ozdaglar, ``Distributed subgradient methods for multi-agent optimization,'' \emph{IEEE Transactions on Automatic Control}, vol.~54, no.~1, pp. 48--61, 2009.

\bibitem{lamport}
L.~Lamport, R.~Shostak, and M.~Pease, ``The {B}yzantine generals problem,'' \emph{ACM Transactions on Programming Languages and Systems}, vol.~4, no.~3, pp. 382--401, 1982.

\bibitem{yin2018byzantine}
D.~Yin, Y.~Chen, R.~Kannan, and P.~Bartlett, ``{B}yzantine-robust distributed learning: Towards optimal statistical rates,'' \emph{International Conference on Machine Learning}, 2018.

\bibitem{chen2017distributed}
Y.~Chen, L.~Su, and J.~Xu, ``Distributed statistical machine learning in adversarial settings: Byzantine gradient descent,'' \emph{ACM on Measurement and Analysis of Computing Systems}, vol.~1, no.~2, pp. 1--25, 2017.

\bibitem{blanchard2017machine}
P.~Blanchard, E.~M. El~Mhamdi, R.~Guerraoui, and J.~Stainer, ``Machine learning with adversaries: {B}yzantine tolerant gradient descent,'' \emph{Advances in Neural Information Processing Systems}, 2017.

\bibitem{karimireddy2021learning}
S.~P. Karimireddy, L.~He, and M.~Jaggi, ``Learning from history for {B}yzantine robust optimization,'' \emph{International Conference on Machine Learning}, pp. 5311--5319, 2021.

\bibitem{xia2019faba}
Q.~Xia, Z.~Tao, Z.~Hao, and Q.~Li, ``{FABA}: an algorithm for fast aggregation against {B}yzantine attacks in distributed neural networks,'' \emph{International Joint Conference on Artificial Intelligence}, 2019.

\bibitem{rsa}
L.~Li, W.~Xu, T.~Chen, G.~B. Giannakis, and Q.~Ling, ``{RSA:} {B}yzantine-robust stochastic aggregation methods for distributed learning from heterogeneous datasets,'' \emph{AAAI Conference on Artificial Intelligence}, 2019.

\bibitem{dong2024}
X.~Dong, Z.~Wu, Q.~Ling, and Z.~Tian, ``Byzantine-robust distributed online learning: Taming adversarial participants in an adversarial environment,'' \emph{IEEE Transactions on Signal Processing}, vol.~72, pp. 235--248, 2024.

\bibitem{wu2022byzantine}
Z.~Wu, T.~Chen, and Q.~Ling, ``Byzantine-resilient decentralized stochastic optimization with robust aggregation rules,'' \emph{IEEE Transactions on Signal Processing}, vol.~71, pp. 3179--3195, 2023.

\bibitem{fang2022bridge}
C.~Fang, Z.~Yang, and W.~U. Bajwa, ``Bridge: Byzantine-resilient decentralized gradient descent,'' \emph{IEEE Transactions on Signal and Information Processing over Networks}, vol.~8, pp. 610--626, 2022.

\bibitem{he2022byzantine}
L.~He, S.~P. Karimireddy, and M.~Jaggi, ``Byzantine-robust decentralized learning via self-centered clipping,'' \emph{arXiv preprint arXiv:2202.01545}, 2022.

\bibitem{Ye2023}
H.~Ye, H.~Zhu, and Q.~Ling, ``On the tradeoff between privacy preservation and {B}yzantine-robustness in decentralized learning,'' \emph{arXiv preprint arXiv:2308.14606}, 2023.

\bibitem{sun2021stability}
T.~Sun, D.~Li, and B.~Wang, ``Stability and generalization of decentralized stochastic gradient descent,'' \emph{AAAI Conference on Artificial Intelligence}, 2021.

\bibitem{deng2023stability}
X.~Deng, T.~Sun, S.~Li, and D.~Li, ``Stability-based generalization analysis of the asynchronous decentralized {SGD},'' \emph{AAAI Conference on Artificial Intelligence}, 2023.

\bibitem{zhu2022topology}
T.~Zhu, F.~He, L.~Zhang, Z.~Niu, M.~Song, and D.~Tao, ``Topology-aware generalization of decentralized {SGD},'' \emph{International Conference on Machine Learning}, 2022.

\bibitem{bars2023improved}
B.~L. Bars, A.~Bellet, M.~Tommasi, K.~Scaman, and G.~Neglia, ``Improved stability and generalization analysis of the decentralized {SGD} algorithm,'' \emph{International Conference on Machine Learning}, 2024.

\bibitem{richards2020graph}
D.~Richards and P.~Rebeschini, ``Graph-dependent implicit regularisation for distributed stochastic subgradient descent,'' \emph{Journal of Machine Learning Research}, vol.~21, no.~34, pp. 1--44, 2020.

\bibitem{taheri2023generalization}
H.~Taheri and C.~Thrampoulidis, ``On generalization of decentralized learning with separable data,'' \emph{International Conference on Artificial Intelligence and Statistics}, 2023.

\bibitem{bousquet2002stability}
O.~Bousquet and A.~Elisseeff, ``Stability and generalization,'' \emph{Journal of Machine Learning Research}, vol.~2, pp. 499--526, 2002.

\bibitem{hardt2016train}
M.~Hardt, B.~Recht, and Y.~Singer, ``Train faster, generalize better: Stability of stochastic gradient descent,'' \emph{International Conference on Machine Learning}, 2016.

\bibitem{sachs2023generalization}
S.~Sachs, T.~van Erven, L.~Hodgkinson, R.~Khanna, and U.~{\c{S}}im{\c{s}}ekli, ``Generalization guarantees via algorithm-dependent {R}ademacher complexity,'' \emph{Annual Conference on Learning Theory}, 2023.

\bibitem{london2017pac}
B.~London, ``A {PAC}-{B}ayesian analysis of randomized learning with application to stochastic gradient descent,'' \emph{Advances in Neural Information Processing Systems}, 2017.

\bibitem{yang2019pac}
{Z. Yang and W. U. Bajwa}, ``{PAC} learning from distributed data in the presence of malicious nodes,'' \emph{International Workshop on Computational Advances in Multi-Sensor Adaptive Processing}, 2019.

\bibitem{xu2017information}
A.~Xu and M.~Raginsky, ``Information-theoretic analysis of generalization capability of learning algorithms,'' \emph{Advances in Neural Information Processing Systems}, 2017.

\bibitem{zhang2022stability}
Y.~Zhang, W.~Zhang, S.~Bald, V.~Pingali, C.~Chen, and M.~Goswami, ``Stability of {SGD}: Tightness analysis and improved bounds,'' \emph{Uncertainty in Artificial Intelligence}, 2022.

\bibitem{bottou2007tradeoffs}
L.~Bottou and O.~Bousquet, ``The tradeoffs of large scale learning,'' \emph{Advances in Neural Information Processing Systems}, 2007.

\bibitem{shalev2010learnability}
S.~Shalev-Shwartz, O.~Shamir, N.~Srebro, and K.~Sridharan, ``Learnability, stability and uniform convergence,'' \emph{Journal of Machine Learning Research}, vol.~11, pp. 2635--2670, 2010.

\bibitem{chen2018stability}
Y.~Chen, C.~Jin, and B.~Yu, ``Stability and convergence trade-off of iterative optimization algorithms,'' \emph{arXiv preprint arXiv:1804.01619}, 2018.

\bibitem{farghly2021time}
T.~Farghly and P.~Rebeschini, ``Time-independent generalization bounds for sgld in non-convex settings,'' \emph{Advances in Neural Information Processing Systems}, 2021.

\bibitem{zhu2024uniform}
L.~Zhu, M.~Gurbuzbalaban, A.~Raj, and U.~Simsekli, ``Uniform-in-time {W}asserstein stability bounds for (noisy) stochastic gradient descent,'' \emph{Advances in Neural Information Processing Systems}, 2023.

\bibitem{hendrickx2024convex}
J.~Hendrickx and A.~Olshevsky, ``Convex {SGD}: Generalization without early stopping,'' \emph{arXiv preprint arXiv:2401.04067}, 2024.

\bibitem{baruch2019little}
G.~Baruch, M.~Baruch, and Y.~Goldberg, ``A little is enough: Circumventing defenses for distributed learning,'' \emph{Advances in Neural Information Processing Systems}, 2019\color{black}.

\bibitem{9462519}
Z.~Wu, H.~Shen, T.~Chen, and Q.~Ling, ``Byzantine-resilient decentralized policy evaluation with linear function approximation,'' \emph{IEEE Transactions on Signal Processing}, vol.~69, pp. 3839--3853, 2021.

\end{thebibliography}

\begin{appendices}

\section{Proof of Theorem \ref{thm-ge}}
\begin{proof} Let $\cS$ and $\cS'$ be two datasets of size $|\RR|Z$, differing in only a single training sample at a certain non-Byzantine agent $i$. In this circumstance, we denote $\cS_i$ and $\cS'_i$ as $\cS_i=\{\xi_{i,1},\cdots,\xi_{i,j},\cdots,\xi_{i,Z} \}$ and $\cS'_i=\{\xi_{i,1},\cdots,\xi'_{i,j},\cdots,\xi_{i,Z} \}$, respectively.
% Taking the notations of Definition \ref{def:on-average},
Following the prior works of \cite{sun2021stability,deng2023stability,bars2023improved}, we consider the generalization error of the average model at a given time $k$.
% In the context of Byzantine-resilient decentralized learning, the output $\mathcal{L}(\cS)$ of a stochastic algorithm $\mathcal{L}$ over dataset $\cS$ is $\bar\vx^k := \frac{1}{|\RR|}\sum_{n \in \RR} \vx_n^k$, the average of all non-Byzantine agents' local models, at a given time $k$.
Denote the output of Algorithm \ref{robust-DSGD} as $\mathcal{L}(\cS)= $ $\bar\vx^k =\frac{1}{|\RR|}\sum_{n \in \RR}\vx_n^k$, $\mathcal{L}(\cS')= \bar\vx'^k =\frac{1}{|\RR|}\sum_{n \in \RR}\vx'^k_n$ and also define $\delta^k = \| \bar\vx^k-\bar\vx'^k \|$.
We discuss the stability of Algorithm \ref{robust-DSGD} in two situations based on whether the virtual mixing matrix is doubly stochastic or row stochastic.
% We also denote by $\{\xi'_{hl}\}_{hl}$ the elements of the data set $\cS^{(ij)}$, i.e. $\xi'_{h,l} = \xi_{h,l}$ for $(h,l) \neq (i,j)$ and $\xi'_{i,j}  \neq \xi_{i,j}$.

\subsection{Doubly Stochastic Virtual Mixing Matrix}
When the virtual mixing matrix is doubly stochastic, it is convenient to analyze $\eta^{k} = \frac{1}{|\RR|} \sum_{n \in \RR} \|\vx_n^{k} - \vx'^{k}_n \|$ instead of $\delta^{k}$, because $\delta^{k} \leq \eta^{k}$. We decompose $\eta^{k+1}$ into three parts, as
% For the two datasets, the  probability of Algorithm \ref{robust-DSGD} selecting the same sample in both $\cS$ and $\cS'$ at time $k$ is $1-\frac{1}{Z}$, we have
\begin{align}
\label{ge-d-1}
%& \delta^{k+1}
%= \|\bar\vx^{k+1}-\bar\vx'^{k+1}\| \\
% =& \|\frac{1}{|\RR|}\sum_{n \in \RR}\vx_n^{k+1} - \frac{1}{|\RR|}\sum_{n \in \RR}\vx'^{k+1}_n\| \nonumber \\
     & \eta^{k+1} = \frac{1}{|\RR|} \sum_{n \in \RR} \|\vx_n^{k+1} - \vx'^{k+1}_n \| \\
\leq & \frac{1}{|\RR|} \sum_{n \in \RR}    \|\vx_n^{k+1} - \hat\vx_n^{k+\frac{1}{2}} \| + \frac{1}{|\RR|} \sum_{n \in \RR}  \|\hat\vx_n^{k+\frac{1}{2}} - \hat\vx_n'^{k+\frac{1}{2}} \| \nonumber\\
&+ \frac{1}{|\RR|} \sum_{n \in \RR}  \|\vx_n'^{k+1} - \hat\vx_n'^{k+\frac{1}{2}} \|.  \nonumber
% \leq & \|\bar\vx^{k+1} - \bar\vx^{k+\frac{1}{2}} \| +  \|\bar\vx^{k+\frac{1}{2}} - \bar\vx'^{k+\frac{1}{2}} \| + \|\bar\vx'^{k+1} - \bar\vx'^{k+\frac{1}{2}} \| \nonumber
\end{align}
% where $\bar\vx^{k+\frac{1}{2}} = \frac{1}{|\RR|}\sum_{n \in \RR}\vx_n^{k+\frac{1}{2}}$ and $\bar\vx'^{k+\frac{1}{2}} = \frac{1}{|\RR|}\sum_{n \in \RR}\vx'^{k+\frac{1}{2}}_n$.
Therein, $\hat\vx_n^{k+\frac{1}{2}} = \sum_{m \in \RR} w_{nm} \vx_m^{k+\frac{1}{2}}$ and $\hat\vx'^{k+\frac{1}{2}}_n = \sum_{m \in \RR}$ $w_{nm} \vx'^{k+\frac{1}{2}}_m$.
For the first term at the right-hand side (RHS) of (\ref{ge-d-1}), we use the contraction property of the robust aggregation rules $ \{ \A_n \}_{n \in  \RR}$ in \eqref{inequality:robustness-of-aggregation-local} to derive
\begin{align}
\label{ge-d-3}
    &\frac{1}{|\RR|} \sum_{n \in \RR} \| \vx_n^{k+1} - \hat\vx_n^{k+\frac{1}{2}} \| \\
    \leq &  \frac{1}{|\RR|} \sum_{n \in \RR} \rho \max_{m \in \RR_n \cup \{n\}} \| \vx_{m}^{k+\frac{1}{2}} - \hat\vx_n^{k+\frac{1}{2}} \| \nonumber \\
    \leq &  \frac{1}{|\RR|} \sum_{n \in \RR} \rho \max_{m \in \RR} \| \vx_{m}^{k+\frac{1}{2}} - \hat\vx_n^{k+\frac{1}{2}} \| \nonumber \\
    \leq &  \frac{1}{|\RR|} \sum_{n \in \RR} \rho \lp \max_{m \in \RR} \| \vx_{m}^{k+\frac{1}{2}} - \bar\vx^{k+\frac{1}{2}} \| + \|  \bar\vx^{k+\frac{1}{2}} - \hat\vx_n^{k+\frac{1}{2}}\|\rp \nonumber \\
    \leq & 2\rho \frac{1}{|\RR|} \sum_{n \in \RR}  \max_{m \in \RR} \| \vx_{m}^{k+\frac{1}{2}} - \bar\vx^{k+\frac{1}{2}} \| \nonumber \\
     \leq & 2\rho \sum_{n \in \RR}   \| \vx_{n}^{k+\frac{1}{2}} - \bar\vx^{k+\frac{1}{2}} \|. \nonumber
\end{align}

According to $\vx_{n}^{k+\frac{1}{2}} = \vx_n^k -\alpha^k \nabla f(\vx^{k}_n; \xi_n^{k}) $ and  $\bar\vx^{k+\frac{1}{2}} = \bar\vx^k -\frac{\alpha^k}{|\RR|} \sum_{m \in \RR} \nabla f(\vx^{k}_m; \xi_m^{k})$, using Assumption \ref{assumption:gradients} we have
\begin{align}
    \label{ge-d-5}
    &  \sum_{n \in \RR}   \| \vx_{n}^{k+\frac{1}{2}} - \bar\vx^{k+\frac{1}{2}} \|  \leq
      \sum_{n \in \RR}  \| \vx_n^k -  \bar\vx^k \|  \\
    & + \alpha^k \sum_{n \in \RR}  \|  \nabla f(\vx^{k}_n; \xi_n^{k}) - \frac{1}{|\RR|} \sum_{m \in \RR} \nabla f(\vx^{k}_m; \xi_m^{k})  \|  \nonumber\\
    \leq & \sum_{n \in \RR}  \| \vx_n^k -  \bar\vx^k \|  + 2 \alpha^k |\RR| M.  \nonumber
\end{align}
The first term at the RHS of (\ref{ge-d-5}) can be bounded by the dis- agreement measure $H^k$ defined in (\ref{hk}), as
\begin{align}
    \label{ge-d-5-1}
    & \sum_{n \in \RR}  \| \vx_n^k -  \bar\vx^k \|   \leq |\RR| \sqrt{H^k}.
\end{align}
Substituting (\ref{ge-d-5}) and (\ref{ge-d-5-1}) into (\ref{ge-d-3}), we obtain
\begin{align}
    \label{ge-d-6}
     \frac{1}{|\RR|} \sum_{n \in \RR} \| \vx_n^{k+1} - \hat\vx_n^{k+\frac{1}{2}} \|
    \leq  2\rho |\RR| ( \sqrt{H^k}+ 2\alpha^k M).
\end{align}
Note that for the third term at the RHS of (\ref{ge-d-1}), an inequality similar to (\ref{ge-d-6}) also holds true.

Now, we analyze the second term at the RHS of (\ref{ge-d-1}). For the non-Byzantine agent $i$, the probability of Algorithm \ref{robust-DSGD} selecting the same training sample from  $\cS_i$ and $\cS'_i$ at time $k$ is $1-\frac{1}{Z}$. According to Lemma 6 in \cite{sun2021stability}, if the loss is strongly convex and Assumption \ref{assumption:Lip} holds, choosing $\alpha^k \leq \frac{1}{L}$ yields
\begin{align}
    \label{ge-d-7}
    % \|\bar\vx^{k+\frac{1}{2}} - \bar\vx'^{k+\frac{1}{2}} \|
    %  \leq   (1-\alpha^k\mu)  \| \bar\vx^k - \bar\vx'^k\|
    & \frac{1}{|\RR|} \sum_{n \in \RR} \|\hat\vx_n^{k+\frac{1}{2}} - \hat\vx_n'^{k+\frac{1}{2}} \|  \\
    = & \frac{1}{|\RR|} \sum_{n \in \RR} \|\sum_{m \in \RR} w_{nm} \lp \vx_m^{k+\frac{1}{2}} - \vx_m'^{k+\frac{1}{2}} \rp\| \nonumber\\
    \leq & \frac{1}{|\RR|} \sum_{n \in \RR} \sum_{m \in \RR} w_{nm}  \| \vx_m^{k+\frac{1}{2}} - \vx_m'^{k+\frac{1}{2}} \| \nonumber\\
    = & \frac{1}{|\RR|} \sum_{n \in \RR} \| \vx_n^{k+\frac{1}{2}} - \vx_n'^{k+\frac{1}{2}} \|  \nonumber\\
    \leq &   (1-\alpha^k\mu) \frac{1}{|\RR|} \sum_{n \in \RR} \| \vx_n^{k} - \vx_n'^{k} \|. \nonumber
\end{align}
Note that the second equality holds true only when the virtual mixing matrix is doubly stochastic. When the virtual mixing matrix is row stochastic but not column stochastic, we are no longer able to analyze the recursion of $\eta^k$ and a different analytical approach is required.
Substituting (\ref{ge-d-6}) and (\ref{ge-d-7}) into (\ref{ge-d-1}), in this case we have
\begin{align}
\label{ge-d-8}
    \E \eta^{k+1} \leq &  (1-\alpha^k\mu)  \E \eta^{k}
    +  4\rho |\RR| (\E \sqrt{H^k}+ 2\alpha^k M).
\end{align}

For the non-Byzantine agent $i$, the probability of Algorithm \ref{robust-DSGD} selecting the different training samples from  $\cS_i$ and $\cS'_i$ at time $k$ is $\frac{1}{Z}$. By Lemma 6 in \cite{sun2021stability}, if the loss is strongly convex and Assumption \ref{assumption:Lip} holds, choosing $\alpha^k \leq \frac{1}{L}$ yields
\begin{align}
    \label{ge-d-9}
  & \frac{1}{|\RR|} \sum_{n \in \RR} \|\hat\vx_n^{k+\frac{1}{2}} - \hat\vx_n'^{k+\frac{1}{2}} \| \\
  \leq  & \frac{1}{|\RR|} \sum_{n \in \RR} \| \vx_n^{k+\frac{1}{2}} - \vx_n'^{k+\frac{1}{2}} \|  \nonumber\\
  = & \frac{1}{|\RR|} ( \|\vx_i^{k+\frac{1}{2}}-\vx'^{k+\frac{1}{2}}_i\| + \sum_{n \in \RR/\{i\}} \| \vx_n^{k+\frac{1}{2}} - \vx_n'^{k+\frac{1}{2}} \| ) \nonumber \\
  \leq & \frac{1}{|\RR|} \| \vx_i^k -\alpha^k \nabla f(\vx^{k}_i; \xi_{i,j})-\vx'^k_i + \alpha^k \nabla f(\vx'^{k}_i; \xi'_{i,j}) \|  \nonumber \\
  & +\frac{1}{|\RR|} \sum_{n \in \RR/\{i\}} \| \vx_n^{k+\frac{1}{2}} - \vx_n'^{k+\frac{1}{2}} \| \nonumber \\
  \leq & \frac{1}{|\RR|} \| \vx_i^k -\alpha^k \nabla f(\vx^{k}_i; \xi_{i,j})-\vx'^k_i + \alpha^k \nabla f(\vx'^{k}_i; \xi_{i,j}) \|  \nonumber \\
  & +\frac{1}{|\RR|} \sum_{n \in \RR/\{i\}} \| \vx_n^{k+\frac{1}{2}} - \vx_n'^{k+\frac{1}{2}} \| \nonumber \\
  & + \frac{\alpha^k}{|\RR|}  \| \nabla f(\vx'^{k}_i; \xi'_{i,j}) - \nabla f(\vx'^{k}_i; \xi_{i,j})\| \nonumber \\
  \leq & (1-\alpha^k\mu) \frac{1}{|\RR|} \sum_{n \in \RR} \| \vx_n^{k} - \vx_n'^{k} \|  + \frac{2 \alpha^k M}{|\RR|}.  \nonumber
\end{align}
Substituting (\ref{ge-d-9}) and (\ref{ge-d-6}) into (\ref{ge-d-1}), in this case we have
\begin{align}
\label{ge-d-10}
    \E \eta^{k+1} \leq &  (1-\alpha^k\mu)  \E \eta^{k} + \frac{2 \alpha^k M}{|\RR|}   \\
    &+  4\rho |\RR| (\E \sqrt{H^k}+ 2\alpha^k M).  \nonumber
\end{align}

Combining the two cases of (\ref{ge-d-8}) and (\ref{ge-d-10}), we can obtain
\begin{align}
\label{ge-d-11}
    \E \eta^{k+1} \leq &  (1-\alpha^k\mu)  \E \eta^{k} + \frac{2 \alpha^k M }{Z|\RR|} \\
    &+  4\rho |\RR| (\E \sqrt{H^k}+ 2\alpha^k M).  \nonumber
\end{align}

According to Lemma \ref{lemma-dm}, if we choose $\alpha^k = \frac{1}{\mu (k+k_0)}$, where $k_0$ is sufficiently large, $\E \sqrt{H^k} $ is bounded by
% \begin{align}
%     \label{ge-d-12}
%     &\E \sqrt{H^{k}}
%         \le  (1-\omega_1^2)^{\frac{k}{2}} \sqrt{H^0}
%         + \frac{\sqrt{2 \omega_2 a_1 }}{\omega_1} \frac{2M}{\mu(k+k_0)}
% \end{align}
\begin{align}
    \label{ge-d-12}
    &\E \sqrt{H^{k}}
        \le  \frac{\sqrt{c}M}{\mu(k+k_0)}.
\end{align}
% where $a_2 = \sqrt{c} $.
% where $a_2 = \frac{2\sqrt{2 \omega_2 a_1 }}{\omega_1} $.
% If we choose $\alpha^k = \frac{1}{\mu (k+k_0)}$, where $\frac{1}{\mu k_0} \leq \frac{1}{L}$ ,
Substituting (\ref{ge-d-12}) into (\ref{ge-d-11}), we then derive
% \begin{align}
% \label{ge-d-13}
%     \E \eta^{k+1}
%     \leq &   (1-\frac{1}{k+k_0})  \E \eta^{k}  +  \frac{2M}{\mu Z|\RR| (k+k_0)}  \\
%     & +  4\rho |\RR|  \lp
%     (a_2)^{\frac{k}{2}} \sqrt{H^0} +
%         \frac{2a_4 M}{\mu(k+k_0)} \rp \nonumber
% \end{align}
\begin{align}
\label{ge-d-13}
    \E \eta^{k+1}
    \leq &    (1-\frac{1}{k+k_0})  \E \eta^{k}  \\
    & +  \frac{2M}{\mu Z|\RR| (k+k_0)}
     +      \frac{4a_2 \rho |\RR| M}{\mu(k+k_0)}.  \nonumber
\end{align}
where
% $a_2 =1-\omega_1^2 \in (0,1) $ and
% $a_3 =\frac{2\sqrt{2 \omega_2 a_1 }}{\omega_1}+2>0   $ are constants.
$a_2 =\sqrt{c}+2>0   $ is a constant.

Using telescopic cancellation on \eqref{ge-d-13} from time $0$ to $k$, we deduce that

% \begin{align}
% \label{ge-d-14}
%     &\E \eta^{K}
%     \leq   \sum_{k=1}^{K-1} \lp \prod_{t=k+1}^{K-1}(1-\frac{1}{t+k_0}) \rp  \cdot \nonumber\\ &\left[ \frac{2M}{\mu Z|\RR| (k+k_0)}
%     +  4\rho |\RR|
%     \lp(a_2)^{\frac{k}{2}} \sqrt{H^0}+
%     \frac{2a_4 M}{\mu(k+k_0)} \rp  \right]  \\
%         \leq&   \sum_{k=1}^{K-1} \frac{k+k_0}{K+k_0-1}  \cdot \nonumber \\ &\left[ \frac{2M}{\mu Z|\RR| (k+k_0)}
%     +  4\rho |\RR|
%     \lp(a_2)^{\frac{k}{2}} \sqrt{H^0} +
%          \frac{2a_4 M}{\mu(k+k_0)}  \rp \right] \nonumber \\
%         &\leq   \frac{2M}{\mu Z|\RR| } + \frac{8a_4\rho |\RR|  M}{\mu}  + \frac{4\rho |\RR| \sqrt{H^0}k_0\sqrt{a_2}}{(K+k_0-1)(1-\sqrt{a_2})}\nonumber \\
%         &+ \frac{4\rho |\RR| \sqrt{H^0}\sqrt{a_2}}{(\sqrt{a_2}-1)^2}\lp\sqrt{a_2}^K+\frac{1}{K+k_0-1}\rp  \nonumber
% \end{align}

\begin{align}
\label{ge-d-14}
    \E \eta^{k}
    \leq&   \sum_{k'=0}^{k-1} \lp \prod_{t=k'+1}^{k-1}(1-\frac{1}{t+k_0})  \rp \cdot \\&  \left[ \frac{2M}{\mu Z|\RR| (k'+k_0)}
    +  \frac{4a_2\rho |\RR| M}{\mu(k'+k_0)}  \right]  \nonumber \\
        \leq&   \sum_{k'=0}^{k-1} \frac{k'+k_0}{k+k_0-1}   \left[ \frac{2M}{\mu Z|\RR| (k'+k_0)}
    +    \frac{4a_2 \rho |\RR| M}{\mu(k'+k_0)}   \right] \nonumber \\
        \leq &   \frac{2M}{\mu Z|\RR| } + \frac{4a_2\rho |\RR|  M}{\mu}.   \nonumber
\end{align}
According to $\E \delta^k \leq \E \eta^k$ and Assumption \ref{assumption:gradients}, we have
\begin{align}
    \label{ge-d-16}
    \E [f(\bar\vx^k;\xi)-f(\bar\vx'^k;\xi)] \leq M \E \delta^k \leq M \E \eta^k.
\end{align}
Substituting (\ref{ge-d-16}) into (\ref{ge-d-14}), according to Lemma \ref{lemma-sta}, we have
\begin{align}
\label{ge-d-17}
    \E_{\cS,\cL} [F(\bar\vx^k)-F_{\cS}(\bar\vx^k)]
    \leq&    \frac{2M^2}{\mu Z|\RR| } + \frac{4a_2\rho |\RR|  M^2}{\mu},
\end{align}
which completes the proof.

\subsection{Row Stochastic Virtual Mixing Matrix}
As previously mentioned, analyzing the recursion of $\eta^k$ is not suitable when the virtual mixing matrix is row stochastic. In this simulation, we directly analyze the recursion of $\delta^{k}$.
We decompose $\delta^{k+1}$ into three parts, as
\begin{align}
\label{ge-1}
& \delta^{k+1}
= \|\bar\vx^{k+1}-\bar\vx'^{k+1}\| \\
% =& \|\frac{1}{|\RR|}\sum_{n \in \RR}\vx_n^{k+1} - \frac{1}{|\RR|}\sum_{n \in \RR}\vx'^{k+1}_n\| \nonumber \\
\leq & \|\bar\vx^{k+1} - \bar\vx^{k+\frac{1}{2}} \| +  \|\bar\vx^{k+\frac{1}{2}} - \bar\vx'^{k+\frac{1}{2}} \| + \|\bar\vx'^{k+1} - \bar\vx'^{k+\frac{1}{2}} \|, \nonumber
\end{align}
where $\bar\vx^{k+\frac{1}{2}} = \frac{1}{|\RR|}\sum_{n \in \RR}\vx_n^{k+\frac{1}{2}}$ and $\bar\vx'^{k+\frac{1}{2}} = \frac{1}{|\RR|}\sum_{n \in \RR}$ $\vx'^{k+\frac{1}{2}}_n$.  For the first term at the RHS of (\ref{ge-1}), we bound it by
\begin{align}
    \label{ge-2}
    &\|\bar\vx^{k+1} - \bar\vx^{k+\frac{1}{2}} \|\\
=& \| \frac{1}{|\RR|}\sum_{n \in \RR}\vx_n^{k+1} - \frac{1}{|\RR|}\sum_{n \in \RR}\vx_n^{k+\frac{1}{2}} \| \nonumber\\
\leq& \| \frac{1}{|\RR|} \sum_{n \in \RR}(\vx_n^{k+1} - \hat\vx_n^{k+\frac{1}{2}}) \| + \| \frac{1}{|\RR|} \sum_{n \in \RR} (\hat\vx_n^{k+\frac{1}{2}} - \vx_n^{k+\frac{1}{2}})\|. \nonumber
\end{align}

The first term at the RHS of (\ref{ge-2}) can be handled in the same way as deriving (\ref{ge-d-6}).
For the second term at the RHS of (\ref{ge-2}), it holds that
\begin{align}
    \label{ge-4}
    &\| \frac{1}{|\RR|} \sum_{n \in \RR} (\hat\vx_n^{k+\frac{1}{2}} - \vx_n^{k+\frac{1}{2}})\| \\
     =&  \lnorm \frac{1}{\vert \RR \vert} \bm{1}^\top (WX^{k+\frac{1}{2}}-\frac{1}{\vert \RR \vert}\bm{1}\bm{1}^{\top}X^{k+\frac{1}{2}})
        \rnorm
          \nonumber\\
        % ==========================
        =& \frac{1}{\vert \RR \vert}  \lnorm (\bm{1}^\top W-\bm{1}^{\top})
        (X^{k+\frac{1}{2}}-\frac{1}{\vert \RR \vert}\bm{1}\bm{1}^{\top}X^{k+\frac{1}{2}})
        \rnorm
         \nonumber\\
        % ==========================
        \le& \frac{1}{\vert \RR \vert} \lnorm W^\top\bm{1} -\bm{1}
        \rnorm
        \lnorm X^{k+\frac{1}{2}}-\frac{1}{\vert \RR \vert}\bm{1}\bm{1}^{\top}X^{k+\frac{1}{2}}
        \rnorm_F
          \nonumber\\
        % ==========================
		=& \frac{\chi}{\sqrt{\vert \RR \vert}}\sum_{n\in\RR} \|  \vx^{k+\frac{1}{2}}_{n}-\bar\vx^{k+\frac{1}{2}}\|.   \nonumber
\end{align}

% According to $\vx_{n}^{k+\frac{1}{2}} = \vx_n^k -\alpha^k \nabla f(\vx^{k}_n; \xi_n^{k}) $ and  $\bar\vx^{k+\frac{1}{2}} = \bar\vx^k -\frac{\alpha^k}{|\RR|} \sum_{m \in \RR} \nabla f(\vx^{k}_m; \xi_m^{k}) $ , we can bound (\ref{ge-3}) and (\ref{ge-4}) by
% \begin{align}
%     \label{ge-5}
%     &  \sum_{n \in \RR}   \| \vx_{n}^{k+\frac{1}{2}} - \bar\vx^{k+\frac{1}{2}} \| \\
%     \leq &  \sum_{n \in \RR} \lp \| \vx_n^k -  \bar\vx^k \| + \alpha^k \|  \nabla f(\vx^{k}_n; \xi_n^{k}) - \frac{1}{|\RR|} \sum_{m \in \RR} \nabla f(\vx^{k}_m; \xi_m^{k})  \| \rp \nonumber\\
%     \leq & \sum_{n \in \RR}  \| \vx_n^k -  \bar\vx^k \|  + 2 \alpha^k |\RR| M  \nonumber
% \end{align}

% We can bound the first term at the RHS of (\ref{ge-5}) by the disagreement measure $H^k$,
% \begin{align}
%     \label{ge-5-1}
%     & \sum_{n \in \RR}  \| \vx_n^k -  \bar\vx^k \|   \leq |\RR| \sqrt{H^k}
% \end{align}

Substituting (\ref{ge-d-6}) and (\ref{ge-4}) into (\ref{ge-2}), we obtain
\begin{align}
    \label{ge-6}
    &\E \|\bar\vx^{k+1} - \bar\vx^{k+\frac{1}{2}} \| \\
    \leq & (2\rho |\RR| +\chi \sqrt{|\RR|})(\E \sqrt{H^k}+ 2\alpha^k M).  \nonumber
\end{align}
Note that for the third term at the RHS of (\ref{ge-1}), an inequality similar to (\ref{ge-6}) also holds true.

Now, we analyze the second term at the RHS of (\ref{ge-1}). For the non-Byzantine agent $i$, the probability of Algorithm \ref{robust-DSGD} selecting the same training sample from  $\cS_i$ and $\cS'_i$ at time $k$ is $1-\frac{1}{Z}$. According to Lemma 6 in \cite{sun2021stability}, if the loss is strongly convex and Assumption \ref{assumption:Lip} holds, choosing $\alpha^k \leq \frac{1}{L}$ yields
\begin{align}
    \label{ge-7}
          \hspace{-2em} & \E \|\bar\vx^{k+\frac{1}{2}} - \bar\vx'^{k+\frac{1}{2}} \|  \\
    \hspace{-1em} = &   \E \|\bar\vx^k -\frac{\alpha^k}{|\RR|} \sum_{m \in \RR} \nabla f(\vx^{k}_m; \xi_m^{k}) -\bar\vx'^k +\frac{\alpha^k}{|\RR|} \sum_{m \in \RR} \nabla f(\vx'^{k}_m; \xi_m^{k}) \|  \nonumber \\
   \hspace{-1em} \leq& \E  \|\bar\vx^k -\alpha^k \nabla f(\bar\vx^{k}) -\bar\vx'^k + \alpha^k\nabla f(\bar\vx'^{k}) \|  \nonumber  \\
   \hspace{-1em} & + \E \|\alpha^k \nabla f(\bar\vx^{k}) - \frac{\alpha^k}{|\RR|} \sum_{m \in \RR} \nabla f(\vx^{k}_m; \xi_m^{k}) \| \nonumber  \\
   \hspace{-1em} & + \E \| \alpha^k\nabla f(\bar\vx'^{k}) - \frac{\alpha^k}{|\RR|} \sum_{m \in \RR} \nabla f(\vx'^{k}_m; \xi_m^{k}) \|   \nonumber  \\
   \hspace{-1em}  \leq&    (1-\alpha^k\mu)  \| \bar\vx^k - \bar\vx'^k\|  + 2 \alpha^k L \E \sqrt{H^k}, \nonumber
\end{align}
where the last inequality holds true because
    \begin{align}
    \label{ge-7-1}
    & \E \|\alpha^k \nabla f(\bar\vx^{k}) - \frac{\alpha^k}{|\RR|} \sum_{m \in \RR} \nabla f(\vx^{k}_m; \xi_m^{k}) \| \\
    \leq & \alpha^k \E \| \frac{1}{|\RR|} \sum_{m \in \RR} (\nabla f(\bar\vx^{k}) - \nabla f(\vx^{k}_m; \xi_m^{k})) \| \nonumber  \\
    \leq & \frac{\alpha^k}{|\RR|} \sum_{m \in \RR} \E \|\nabla f(\bar\vx^{k}) - \nabla f(\vx^{k}_m; \xi_m^{k}) \| \nonumber  \\
    \leq &  \frac{\alpha^k L}{|\RR|} \sum_{m \in \RR} \E \|\bar\vx^{k} - \vx^{k}_m\| \nonumber  \\
    \leq & \alpha^k L \E \sqrt{H^k}. \nonumber
\end{align}
Substituting (\ref{ge-6}) and (\ref{ge-7}) into (\ref{ge-1}), in this case we have
\begin{align}
\label{ge-8}
    \E \delta^{k+1} \leq &  (1-\alpha^k\mu)  \E \delta^{k} + 2 \alpha^k L \E \sqrt{H^k} \\
    &+  (4\rho |\RR| +2\chi\sqrt{|\RR|})(\E \sqrt{H^k}+ 2\alpha^k M).  \nonumber
\end{align}
% \begin{align}
%     \label{ge-6}
%     \| \frac{1}{|\RR|} \sum_{n \in \RR} (\hat\vx_n^{k+\frac{1}{2}} - \vx_n^{k+\frac{1}{2}})\|
%     \leq  \chi\sqrt{|\RR|}(H^k + 2 \alpha^k  M)
% \end{align}

For the non-Byzantine agent $i$, the probability of Algorithm \ref{robust-DSGD} selecting the different training samples from  $\cS_i$ and $\cS'_i$ at time $k$ is $\frac{1}{Z}$. By Lemma 6 in \cite{sun2021stability}, if the loss is strongly convex and Assumption \ref{assumption:Lip} holds, choosing $\alpha^k \leq \frac{1}{L}$ yields
\begin{align}
    \label{ge-9}
   &\|\bar\vx^{k+\frac{1}{2}} - \bar\vx'^{k+\frac{1}{2}} \|   \\
    =& \frac{1}{|\RR|} \|\vx_i^{k+\frac{1}{2}}-\vx'^{k+\frac{1}{2}}_i +  \sum_{n \in \RR/\{i\}}(\vx^{k+\frac{1}{2}}_n-\vx'^{k+\frac{1}{2}}_n)\|  \nonumber \\
    =& \frac{1}{|\RR|} \| \vx_i^k -\alpha^k \nabla f(\vx^{k}_i; \xi_{i,j})-\vx'^k_i + \alpha^k \nabla f(\vx'^{k}_i; \xi'_{i,j}) \nonumber\\
    &+  \sum_{n \in \RR/\{i\}}(\vx^{k+\frac{1}{2}}_n-\vx'^{k+\frac{1}{2}}_n) \|  \nonumber\\
    \leq&  \frac{1}{|\RR|} \| \vx_i^k -\alpha^k \nabla f(\vx^{k}_i; \xi_{i,j})-\vx'^k_i + \alpha^k \nabla f(\vx'^{k}_i; \xi_{i,j}) \nonumber \\
    & +  \sum_{n \in \RR/\{i\}}(\vx^{k+\frac{1}{2}}_n-\vx'^{k+\frac{1}{2}}_n) \|  \nonumber \\
    & + \frac{\alpha^k}{|\RR|}  \| \nabla f(\vx'^{k}_i; \xi'_{i,j}) - \nabla f(\vx'^{k}_i; \xi_{i,j})\| \nonumber\\
    \leq&  (1-\alpha^k\mu)  \| \bar\vx^k - \bar\vx'^k\| + 2 \alpha^k L \E \sqrt{H^k} + \frac{2 \alpha^k M}{|\RR|}.  \nonumber
\end{align}
Substituting (\ref{ge-9}) and (\ref{ge-7}) into (\ref{ge-1}), in this case we have
\begin{align}
\label{ge-10}
    \E \delta^{k+1} \leq &  (1-\alpha^k\mu)  \E \delta^{k} + 2 \alpha^k L \E \sqrt{H^k} + \frac{2 \alpha^k M}{|\RR|}   \\
    &+  (4\rho |\RR| +2\chi\sqrt{|\RR|})(\E \sqrt{H^k}+ 2\alpha^k M).  \nonumber
\end{align}

Combining the two cases of (\ref{ge-8}) and (\ref{ge-10}), we can obtain
\begin{align}
\label{ge-11}
    \E \delta^{k+1} \leq &  (1-\alpha^k\mu)  \E \delta^{k} + 2 \alpha^k L \E \sqrt{H^k} + \frac{2 \alpha^k M }{Z|\RR|} \\
    &+  (4\rho |\RR| +2\chi\sqrt{|\RR|})(\E \sqrt{H^k}+ 2\alpha^k M).  \nonumber
\end{align}

% According to (\ref{dm-12}), we can bound $\E \sqrt{H^k} $ by
% % \begin{align}
% %     \label{ge-12}
% %     &\E \sqrt{H^{k}}
% %         \le  (1-\omega_1^2)^{\frac{k}{2}} \sqrt{H^0}
% %         + \frac{\sqrt{2 \omega_2 a_1 }}{\omega_1} \frac{2M}{\mu(k+k_0)}
% % \end{align}
% \begin{align}
%     \label{ge-12}
%     &\E \sqrt{H^{k}}
%         \le   \frac{\sqrt{2 \omega_2 a_1 }}{\omega_1} \frac{2M}{\mu(k+k_0)}
% \end{align}

% where $a_2 = 1-\omega_1^2 \in (0,1)$  is constant
%and $a_3 = \frac{\sqrt{2 \omega_2 a_1 }}{\omega_1}$ are constants.
If we choose $\alpha^k = \frac{1}{\mu (k+k_0)}$, where $k_0$ is sufficiently large, substituting (\ref{ge-d-12}) into (\ref{ge-11}), we then derive
% \begin{align}
% \label{ge-13}
%     &\E \delta^{k+1}
%     \leq   (1-\frac{1}{k+k_0})  \E \delta^{k} +  \frac{2M}{\mu Z|\RR| (k+k_0)} \\
%     & + \frac{2L}{\mu (k+k_0)}  \lp(a_2)^{\frac{k}{2}} \sqrt{H^0}
%         + \frac{2a_3 M}{\mu(k+k_0)} \rp  \nonumber \\
%     &+  (4\rho |\RR| +2\chi\sqrt{|\RR|}) \lp(a_2)^{\frac{k}{2}} \sqrt{H^0}
%         + \frac{2a_4 M}{\mu(k+k_0)} \rp  \nonumber
% \end{align}
\begin{align}
\label{ge-13}
    &\E \delta^{k+1}
    \leq   (1-\frac{1}{k+k_0})  \E \delta^{k} +  \frac{2M}{\mu Z|\RR| (k+k_0)} \\
    &  + \frac{2 \sqrt{c} M L}{\mu^2 (k+k_0)^2}   +  (4\rho |\RR| +2\chi\sqrt{|\RR|})
         \frac{a_2 M}{\mu(k+k_0)}. \nonumber
\end{align}
% where  $a_2 =1-\omega_1^2 \in (0,1) $, $a_3 =\frac{\sqrt{2 \omega_2 a_1 }}{\omega_1}  $ and $a_4 =\frac{\sqrt{2 \omega_2 a_1 }}{\omega_1} +1 $  are constants.
% where  $a_2 =1-\omega_1^2 \in (0,1) $, $a_3 =\frac{2\sqrt{2 \omega_2 a_1 }}{\omega_1}  $ and $a_2 =\frac{2\sqrt{2 \omega_2 a_1 }}{\omega_1} +2 $  are constants.

Using telescopic cancellation on \eqref{ge-13} from time $0$ to $k$, we deduce that
% \begin{align}
% \label{ge-14}
%     &\E \delta^{k} \\
%     \leq &   \sum_{k=1}^{K-1} \lp \prod_{t=k+1}^{K-1}(1-\frac{1}{t+k_0}) \rp  \Bigg[ \frac{2a_2 M L}{\mu^2 (k+k_0)^2}   \nonumber\\ & +\frac{2M}{\mu Z|\RR| (k+k_0)}
%     +  (4\rho |\RR| +2\chi\sqrt{|\RR|})
%          \frac{a_3 M}{\mu(k+k_0)}   \Bigg] \nonumber \\
%         \leq &   \sum_{k=1}^{K-1} \frac{k+k_0}{K+k_0-1}   \Bigg[ \frac{2a_2 M L}{\mu^2 (k+k_0)^2} \nonumber\\ & +\frac{2M}{\mu Z|\RR| (k+k_0)}
%     +  (4\rho |\RR| +2\chi\sqrt{|\RR|})
%         \frac{a_3 M}{\mu(k+k_0)}   \Bigg] \nonumber \\
%         &\leq   \frac{2M}{\mu Z|\RR| } + \frac{a_3(4\rho |\RR| +2\chi\sqrt{|\RR|}) M}{\mu} + \frac{2L\sqrt{H^0}\sqrt{a_2}}{\mu(1-\sqrt{a_2})(K+k_0-1) }\nonumber \\
%         &+ \frac{(4\rho |\RR| +2\chi\sqrt{|\RR|})\sqrt{H^0}\sqrt{a_2}}{(\sqrt{a_2}-1)^2}\lp\sqrt{a_2}^K+\frac{1}{K+k_0-1}\rp  \nonumber \\
%        & + \frac{(4\rho |\RR| +2\chi\sqrt{|\RR|})\sqrt{H^0}k_0\sqrt{a_2}}{(K+k_0-1)(1-\sqrt{a_2})} + \frac{4a_3LMln(K+k_0)}{\mu^2(K+k_0-1) }  \nonumber
% \end{align}
\begin{align}
\label{ge-14}
    &\E \delta^{k}
    \leq    \sum_{k'=0}^{k-1} \lp \prod_{t=k'+1}^{k-1}(1-\frac{1}{t+k_0}) \rp  \Bigg[ \frac{2\sqrt{c} M L}{\mu^2 (k'+k_0)^2}   \\ & +\frac{2M}{\mu Z|\RR| (k'+k_0)}
    +  (4\rho |\RR| +2\chi\sqrt{|\RR|})
         \frac{a_2 M}{\mu(k'+k_0)}   \Bigg] \nonumber \\
        \leq &   \sum_{k'=0}^{k-1} \frac{k'+k_0}{k+k_0-1}   \Bigg[ \frac{2\sqrt{c} M L}{\mu^2 (k'+k_0)^2} \nonumber\\ & +\frac{2M}{\mu Z|\RR| (k'+k_0)}
    +  (4\rho |\RR| +2\chi\sqrt{|\RR|})
        \frac{a_2 M}{\mu(k'+k_0)}   \Bigg] \nonumber \\
        &\leq   \frac{2M}{\mu Z|\RR| } + \frac{a_2(4\rho |\RR| +2\chi\sqrt{|\RR|}) M}{\mu}
       + \frac{2\sqrt{c} MLln(k+k_0)}{\mu^2(k+k_0-1) }.  \nonumber
\end{align}
% The last inequality holds true because
% \begin{align}
%     \label{ge-15}
%     & \sum_{k=1}^{K-1} (k+k_0) (\sqrt{a_2})^k \\
%     =&  \sum_{k=1}^{K-1} k(\sqrt{a_2})^k + k_0 \sum_{k=1}^{K-1} (\sqrt{a_2})^k   \nonumber \\
%     = & \frac{\sqrt{a_2}\lp(K-1)\sqrt{a_2}^K-K\sqrt{a_2}^{K-1}+1\rp}{(\sqrt{a_2}-1)^2} + \frac{k_0\sqrt{a_2}(1-\sqrt{a_2}^{K-1})}{1-\sqrt{a_2}} \nonumber \\
%     \leq & \frac{\sqrt{a_2}\lp(K-1)\sqrt{a_2}^K+1\rp}{(\sqrt{a_2}-1)^2} + \frac{k_0\sqrt{a_2}}{1-\sqrt{a_2}}  \nonumber
% \end{align}
According to Assumption \ref{assumption:gradients}, we have
\begin{align}
    \label{ge-16}
    \E [f(\bar\vx^k;\xi)-f(\bar\vx'^k;\xi)] \leq M \E \delta^k.
\end{align}
Substituting (\ref{ge-16}) into (\ref{ge-14}), according to Lemma \ref{lemma-sta}, we have
\begin{align}
    \label{ge-17}
    &\E_{\cS,\cL} [F(\bar\vx^k)-F_{\cS}(\bar\vx^k)]  \\
    \leq &   \frac{2M^2}{\mu Z|\RR| } + \frac{a_2(4\rho |\RR| +2\chi\sqrt{|\RR|}) M^2}{\mu}
         + \frac{2\sqrt{c}M^2Lln(k+k_0)}{\mu^2(k+k_0-1) }, \nonumber
\end{align}
which completes the proof.
\end{proof}

\section{Proof of Theorem \ref{thm-ge-c}}
\subsection{Doubly Stochastic Virtual Mixing Matrix}
\begin{proof}
When the virtual mixing matrix is doubly stochastic, it is convenient to analyze $\eta^{k} = \frac{1}{|\RR|} \sum_{n \in \RR} \|\vx_n^{k} - \vx'^{k}_n \|$ instead of $\delta^{k}$, as $\delta^{k} \leq \eta^{k}$. We decompose $\eta^{k+1}$ into three parts as shown in (\ref{ge-d-1}) and analyze the first and third terms following the same approach as in (\ref{ge-d-3})--(\ref{ge-d-6}).

Now, we analyze the second term at the RHS of (\ref{ge-d-1}). For the non-Byzantine agent $i$, the probability of Algorithm \ref{robust-DSGD} selecting the same training sample from  $\cS_i$ and $\cS'_i$ at time $k$ is $1-\frac{1}{Z}$. According to Lemma 6 in \cite{sun2021stability}, if the loss is convex and Assumption \ref{assumption:Lip} holds, choosing $\alpha^k \leq \frac{2}{L}$ yields
\begin{align}
\label{ge-dc-1}
& \frac{1}{|\RR|} \sum_{n \in \RR} \|\hat\vx_n^{k+\frac{1}{2}} - \hat\vx_n'^{k+\frac{1}{2}} \|  \\
    = & \frac{1}{|\RR|} \sum_{n \in \RR} \|\sum_{m \in \RR} w_{nm} \lp \vx_m^{k+\frac{1}{2}} - \vx_m'^{k+\frac{1}{2}} \rp\| \nonumber\\
    \leq & \frac{1}{|\RR|} \sum_{n \in \RR} \sum_{m \in \RR} w_{nm}  \| \vx_m^{k+\frac{1}{2}} - \vx_m'^{k+\frac{1}{2}} \| \nonumber\\
    = & \frac{1}{|\RR|} \sum_{n \in \RR} \| \vx_n^{k+\frac{1}{2}} - \vx_n'^{k+\frac{1}{2}} \|  \nonumber\\
    \leq &   \frac{1}{|\RR|} \sum_{n \in \RR} \| \vx_n^{k} - \vx_n'^{k} \|. \nonumber
\end{align}
Note that the second equality holds true only when the virtual mixing matrix is doubly stochastic. When the virtual mixing matrix is row stochastic but not column stochastic, we are no longer able to analyze the recursion of $\eta^k$ and a different analytical approach is required. Substituting (\ref{ge-d-6}) and (\ref{ge-dc-1}) into (\ref{ge-d-1}), in this case we have
\begin{align}
\label{ge-dc-2}
   \E \eta^{k+1} \leq &  \E \eta^{k}
    +  4\rho |\RR| (\E \sqrt{H^k}+ 2\alpha^k M).
\end{align}

For the non-Byzantine agent $i$, the probability of Algorithm \ref{robust-DSGD} selecting the different training samples from  $\cS_i$ and $\cS'_i$ at time $k$ is $\frac{1}{Z}$. By Lemma 6 in \cite{sun2021stability}, if the loss is convex and Assumption \ref{assumption:Lip} holds, choosing $\alpha^k \leq \frac{2}{L}$ yields

\begin{align}
    \label{ge-dc-3}
  & \frac{1}{|\RR|} \sum_{n \in \RR} \|\hat\vx_n^{k+\frac{1}{2}} - \hat\vx_n'^{k+\frac{1}{2}} \| \\
  \leq  & \frac{1}{|\RR|} \sum_{n \in \RR} \| \vx_n^{k+\frac{1}{2}} - \vx_n'^{k+\frac{1}{2}} \|  \nonumber\\
  % = & \frac{1}{|\RR|} ( \|\vx_i^{k+\frac{1}{2}}-\vx'^{k+\frac{1}{2}}_i\| + \sum_{n \in \RR/\{i\}} \| \vx_n^{k+\frac{1}{2}} - \vx_n'^{k+\frac{1}{2}} \| ) \nonumber \\
  % \leq & \frac{1}{|\RR|} \| \vx_i^k -\alpha^k \nabla f(\vx^{k}_i; \xi_{i,j})-\vx'^k_i + \alpha^k \nabla f(\vx'^{k}_i; \xi'_{i,j}) \|  \nonumber \\
  % & +\frac{1}{|\RR|} \sum_{n \in \RR/\{i\}} \| \vx_n^{k+\frac{1}{2}} - \vx_n'^{k+\frac{1}{2}} \| \nonumber \\
  % \leq & \frac{1}{|\RR|} \| \vx_i^k -\alpha^k \nabla f(\vx^{k}_i; \xi_{i,j})-\vx'^k_i + \alpha^k \nabla f(\vx'^{k}_i; \xi_{i,j}) \|  \nonumber \\
  % & +\frac{1}{|\RR|} \sum_{n \in \RR/\{i\}} \| \vx_n^{k+\frac{1}{2}} - \vx_n'^{k+\frac{1}{2}} \| \nonumber \\
  % & + \frac{\alpha^k}{|\RR|}  \| \nabla f(\vx'^{k}_i; \xi'_{i,j}) - \nabla f(\vx'^{k}_i; \xi_{i,j})\| \nonumber \\
  \leq & \frac{1}{|\RR|} \sum_{n \in \RR} \| \vx_n^{k} - \vx_n'^{k} \|  + \frac{2 \alpha^k M}{|\RR|}.  \nonumber
\end{align}
Substituting (\ref{ge-d-6}) and (\ref{ge-dc-3}) into (\ref{ge-d-1}), in this case we have
\begin{align}
\label{ge-dc-4}
    \E \eta^{k+1} \leq   \E \eta^{k} + \frac{2 \alpha^k M}{|\RR|}
    +  4\rho |\RR| (\E \sqrt{H^k}+ 2\alpha^k M).
\end{align}

Combining the two cases of (\ref{ge-dc-2}) and (\ref{ge-dc-4}), we can obtain
\begin{align}
\label{ge-dc-5}
    \E \eta^{k+1} \leq &   \E \eta^{k} + \frac{2 \alpha^k M }{Z|\RR|}
    +  4\rho |\RR| (\E \sqrt{H^k}+ 2\alpha^k M).
\end{align}

If we choose $\alpha^k = \frac{1}{ k+k_0}$, where $ k_0$ is sufficiently large, we then derive
\begin{align}
\label{ge-dc-6}
    \E \eta^{k+1}
    \leq &      \E \eta^{k}
     +  \frac{2M}{ Z|\RR| (k+k_0)}
     +      \frac{4a_2 \rho |\RR| M}{k+k_0}.
\end{align}

Using telescopic cancellation on \eqref{ge-dc-6} from time $0$ to $k$, we deduce that

\begin{align}
\label{ge-dc-7}
    \E \eta^{k}
    \leq & \lp \frac{2M}{ Z|\RR|} + 4a_2 \rho |\RR| M \rp \sum_{k'=0}^{k-1} \frac{1}{k'+k_0}  \\
    % & +   (4\rho |\RR| +2\chi\sqrt{|\RR|}) H^0 \sum_{k=1}^{K-1} (a_2)^{\frac{k}{2}} \nonumber \\
    \leq & \lp \frac{2M}{ Z|\RR|} + 4a_2 \rho |\RR| M \rp  ln(k+k_0). \nonumber
    % &+ (4\rho |\RR| +2\chi\sqrt{|\RR|}) \frac{\sqrt{a_2}}{1-\sqrt{a_2}} H^0
\end{align}
Substituting  (\ref{ge-d-16}) into (\ref{ge-dc-7}), according to Lemma \ref{lemma-sta}, we have
\begin{align}
\label{ge-dc-8}
    & \E_{\cS,\cL} [F(\bar\vx^k)-F_{\cS}(\bar\vx^k)] \\
    \leq  &\lp \frac{2M^2}{ Z|\RR| } + 4a_2\rho |\RR|  M^2  \rp  ln(k+k_0),\nonumber
\end{align}
which completes the proof.

\subsection{Row Stochastic Virtual Mixing Matrix}
As previously mentioned, analyzing the recursion of $\eta^k$ is not suitable when the virtual mixing matrix is row stochastic. In this simulation, we directly analyze the recursion of $\delta^{k}$. We decompose $\delta^{k+1}$ into three parts as shown in (\ref{ge-1}), and can analyze the first and third terms following the same approach as in (\ref{ge-2})--(\ref{ge-6}).

Now, we analyze the second term at the RHS of (\ref{ge-1}). For the non-Byzantine agent $i$, the probability of Algorithm \ref{robust-DSGD} selecting the same training sample from  $\cS_i$ and $\cS'_i$ at time $k$ is $1-\frac{1}{Z}$. According to Lemma 6 in \cite{sun2021stability}, if the loss is convex and Assumption \ref{assumption:Lip} holds, choosing $\alpha^k \leq \frac{2}{L}$ yields
\begin{align}
    \label{ge-c-1}
    \|\bar\vx^{k+\frac{1}{2}} - \bar\vx'^{k+\frac{1}{2}} \|
     \leq    \| \bar\vx^k - \bar\vx'^k\|  + 2 \alpha^k L \E \sqrt{H^k}.
\end{align}
Substituting (\ref{ge-6}) and (\ref{ge-c-1}) into (\ref{ge-1}), in this case we have
\begin{align}
\label{ge-c-2}
     \E \delta^{k+1}
    \leq &\E \delta^{k} + 2 \alpha^k L \E \sqrt{H^k}  \\
  &+  (4\rho |\RR| +2\chi\sqrt{|\RR|})(\E \sqrt{H^k}+ 2\alpha^k M).  \nonumber
\end{align}

For the non-Byzantine agent $i$, the probability of Algorithm \ref{robust-DSGD} selecting the different training samples from  $\cS_i$ and $\cS'_i$ at time $k$ is $\frac{1}{Z}$. By Lemma 6 in \cite{sun2021stability}, if the loss is convex and Assumption \ref{assumption:Lip} holds, choosing $\alpha^k \leq \frac{2}{L}$ yields
\begin{align}
     \label{ge-c-3}
   &\|\bar\vx^{k+\frac{1}{2}} - \bar\vx'^{k+\frac{1}{2}} \|   \\
    % =& \frac{1}{|\RR|} \|\vx_i^{k+\frac{1}{2}}-\vx'^{k+\frac{1}{2}}_i +  \sum_{n \in \RR/\{i\}}(\vx^{k+\frac{1}{2}}_n-\vx'^{k+\frac{1}{2}}_n)\|  \nonumber \\
    % =& \frac{1}{|\RR|} \| \vx_i^k -\alpha^k \nabla f(\vx^{k}_n; \xi_{i,j})-\vx'^k + \alpha^k \nabla f(\vx^{k}_n; \xi'_{i,j}) \nonumber\\
    % &+  \sum_{n \in \RR/\{i\}}(\vx^{k+\frac{1}{2}}_n-\vx'^{k+\frac{1}{2}}_n) \|  \nonumber\\
    % \leq&  \frac{1}{|\RR|} \| \vx_i^k -\alpha^k \nabla f(\vx^{k}_n; \xi_{i,j})-\vx'^k + \alpha^k \nabla f(\vx^{k}_n; \xi_{i,j}) \nonumber \\
    % & +  \sum_{n \in \RR/\{i\}}(\vx^{k+\frac{1}{2}}_n-\vx'^{k+\frac{1}{2}}_n) \|  \nonumber \\
    % & + \frac{\alpha^k}{|\RR|}  \| \nabla f(\vx^{k}_n; \xi'_{i,j}) - \nabla f(\vx^{k}_n; \xi_{i,j})\| \nonumber\\
    \leq&  \| \bar\vx^k - \bar\vx'^k\| + \frac{2 \alpha^k M}{|\RR|} + 2 \alpha^k L \E \sqrt{H^k} . \nonumber
\end{align}
Substituting (\ref{ge-6}) and (\ref{ge-c-3}) into (\ref{ge-1}), in this case we have
\begin{align}
\label{ge-c-4}
    \E \delta^{k+1} \leq &  \E \delta^{k} + \frac{2 \alpha^k M }{|\RR|} + 2 \alpha^k L \E \sqrt{H^k}\\
    &+  (4\rho |\RR| +2\chi\sqrt{|\RR|})(\E \sqrt{H^k}+ 2\alpha^k M) . \nonumber
\end{align}

Combining the two cases of (\ref{ge-c-2}) and (\ref{ge-c-4}), we can obtain
\begin{align}
\label{ge-c-5}
    \E \delta^{k+1} \leq &  \E \delta^{k} + \frac{2 \alpha^k M }{Z|\RR|} + 2 \alpha^k L \E \sqrt{H^k} \\
    &+  (4\rho |\RR| +2\chi\sqrt{|\RR|})(\E \sqrt{H^k}+ 2\alpha^k M).  \nonumber
\end{align}

If we choose $\alpha^k = \frac{1}{ k+k_0}$, where $k_0$ is sufficiently large,  we then derive
\begin{align}
\label{ge-c-6}
    \E \delta^{k+1}
    \leq&    \E \delta^{k} + \frac{2 \sqrt{c} M L}{ (k+k_0)^2}
    + \frac{2M}{ Z|\RR| (k+k_0)}  \\ &+  (4\rho |\RR| +2\chi\sqrt{|\RR|})
    % \lp(a_2)^{\frac{k}{2}} H^0 +
    \frac{a_2 M}{k+k_0}.  \nonumber
\end{align}
% where
% $a_2 =1-\omega_1^2 \in (0,1) $ and
% $a_2 =\frac{\sqrt{2 \omega_2 a_1 }}{\omega_1} +1 $ is a constant.

 Using telescopic cancellation on \eqref{ge-c-6} from time $0$ to $k$, we deduce that

\begin{align}
\label{ge-c-7}
    &\E \delta^{k} \\
    \leq & \lp \frac{2M}{ Z|\RR|} + a_2(4\rho |\RR| +2\chi\sqrt{|\RR|})M \rp \sum_{k'=0}^{k-1} \frac{1}{k'+k_0} \nonumber \\
    & + 2\sqrt{c} M L \sum_{k'=0}^{k-1} \frac{1}{(k'+k_0)^2} \nonumber \\
    % & +   (4\rho |\RR| +2\chi\sqrt{|\RR|}) H^0 \sum_{k=1}^{K-1} (a_2)^{\frac{k}{2}} \nonumber \\
    \leq & \lp \frac{2M}{ Z|\RR|} + a_2(4\rho |\RR| +2\chi\sqrt{|\RR|})M \rp  ln(k+k_0) \nonumber \\
    & + 2\sqrt{c} M L. \nonumber
    % &+ (4\rho |\RR| +2\chi\sqrt{|\RR|}) \frac{\sqrt{a_2}}{1-\sqrt{a_2}} H^0
\end{align}
Substituting (\ref{ge-16}) into (\ref{ge-c-7}), according to Lemma \ref{lemma-sta}, we have
\begin{align}
    \label{ge-c-10}
   & \E_{\cS,\cL} [F(\bar\vx^k)-F_{\cS}(\bar\vx^k)]
    \leq 2 \sqrt{c} M^2 L  \\ &+ \lp \frac{2M^2}{ Z|\RR|} + a_2(4\rho |\RR| +2\chi\sqrt{|\RR|})M^2 \rp  ln(k+k_0),  \nonumber
    % &+ (4\rho |\RR| +2\chi\sqrt{|\RR|}) \frac{\sqrt{a_2}}{1-\sqrt{a_2}} H^0 M
\end{align}
which completes the proof.
\end{proof}

\section{Proof of Theorem \ref{thm-ge-n}}
\subsection{Doubly Stochastic Virtual Mixing Matrix}
\begin{proof}
When the virtual mixing matrix is doubly stochastic, it is convenient to analyze $\eta^{k} = \frac{1}{|\RR|} \sum_{n \in \RR} \|\vx_n^{k} - \vx'^{k}_n \|$ instead of $\delta^{k}$, as $\delta^{k} \leq \eta^{k}$. We decompose $\eta^{k+1}$ into three parts as shown in (\ref{ge-d-1}) and analyze the first and third terms following the same approach as in (\ref{ge-d-3})--(\ref{ge-d-6}).

Now, we analyze the second term at the RHS of (\ref{ge-d-1}). For the non-Byzantine agent $i$, the probability of Algorithm \ref{robust-DSGD} selecting the same training sample from  $\cS_i$ and $\cS'_i$ at time $k$ is $1-\frac{1}{Z}$. By Lemma 6 in \cite{sun2021stability}, if Assumption \ref{assumption:Lip} holds, we have

\begin{align}
\label{ge-dn-1}
& \frac{1}{|\RR|} \sum_{n \in \RR} \|\hat\vx_n^{k+\frac{1}{2}} - \hat\vx_n'^{k+\frac{1}{2}} \|  \\
    = & \frac{1}{|\RR|} \sum_{n \in \RR} \|\sum_{m \in \RR} w_{nm} \lp \vx_m^{k+\frac{1}{2}} - \vx_m'^{k+\frac{1}{2}} \rp\| \nonumber\\
    \leq & \frac{1}{|\RR|} \sum_{n \in \RR} \sum_{m \in \RR} w_{nm}  \| \vx_m^{k+\frac{1}{2}} - \vx_m'^{k+\frac{1}{2}} \| \nonumber\\
    = & \frac{1}{|\RR|} \sum_{n \in \RR} \| \vx_n^{k+\frac{1}{2}} - \vx_n'^{k+\frac{1}{2}} \|  \nonumber\\
    \leq &   \frac{1+\alpha^k L}{|\RR|} \sum_{n \in \RR} \| \vx_n^{k} - \vx_n'^{k} \| .\nonumber
\end{align}
Note that the second equality holds true only when the virtual mixing matrix is doubly stochastic. When the virtual mixing matrix is row stochastic but not column stochastic, we are no longer able to analyze the recursion of $\eta^k$ and a different analytical approach is required. Substituting (\ref{ge-d-6}) and (\ref{ge-dn-1}) into (\ref{ge-d-1}), in this case we have
\begin{align}
\label{ge-dn-2}
   \E \eta^{k+1} \leq &  (1+\alpha^k L) \E \eta^{k}
    +  4\rho |\RR| (\E \sqrt{H^k}+ 2\alpha^k M).
\end{align}

For the non-Byzantine agent $i$, the probability of Algorithm \ref{robust-DSGD} selecting the different training samples from  $\cS_i$ and $\cS'_i$ at time $k$ is $\frac{1}{Z}$. By Lemma 6 in \cite{sun2021stability}, if Assumption \ref{assumption:Lip} holds, we have

\begin{align}
    \label{ge-dn-3}
  & \frac{1}{|\RR|} \sum_{n \in \RR} \|\hat\vx_n^{k+\frac{1}{2}} - \hat\vx_n'^{k+\frac{1}{2}} \| \\
  \leq  & \frac{1}{|\RR|} \sum_{n \in \RR} \| \vx_n^{k+\frac{1}{2}} - \vx_n'^{k+\frac{1}{2}} \|  \nonumber\\
  % = & \frac{1}{|\RR|} ( \|\vx_i^{k+\frac{1}{2}}-\vx'^{k+\frac{1}{2}}_i\| + \sum_{n \in \RR/\{i\}} \| \vx_n^{k+\frac{1}{2}} - \vx_n'^{k+\frac{1}{2}} \| ) \nonumber \\
  % \leq & \frac{1}{|\RR|} \| \vx_i^k -\alpha^k \nabla f(\vx^{k}_i; \xi_{i,j})-\vx'^k_i + \alpha^k \nabla f(\vx'^{k}_i; \xi'_{i,j}) \|  \nonumber \\
  % & +\frac{1}{|\RR|} \sum_{n \in \RR/\{i\}} \| \vx_n^{k+\frac{1}{2}} - \vx_n'^{k+\frac{1}{2}} \| \nonumber \\
  % \leq & \frac{1}{|\RR|} \| \vx_i^k -\alpha^k \nabla f(\vx^{k}_i; \xi_{i,j})-\vx'^k_i + \alpha^k \nabla f(\vx'^{k}_i; \xi_{i,j}) \|  \nonumber \\
  % & +\frac{1}{|\RR|} \sum_{n \in \RR/\{i\}} \| \vx_n^{k+\frac{1}{2}} - \vx_n'^{k+\frac{1}{2}} \| \nonumber \\
  % & + \frac{\alpha^k}{|\RR|}  \| \nabla f(\vx'^{k}_i; \xi'_{i,j}) - \nabla f(\vx'^{k}_i; \xi_{i,j})\| \nonumber \\
  \leq & \frac{1+\alpha^k L}{|\RR|} \sum_{n \in \RR} \| \vx_n^{k} - \vx_n'^{k} \|  + \frac{2 \alpha^k M}{|\RR|} . \nonumber
\end{align}
Substituting (\ref{ge-d-6}) and (\ref{ge-dn-3}) into (\ref{ge-d-1}), in this case we have
\begin{align}
\label{ge-dn-4}
    \E \eta^{k+1} \leq&   (1+\alpha^k L) \E \eta^{k} + \frac{2 \alpha^k M}{|\RR|}   \\
    &+  4\rho |\RR| (\E \sqrt{H^k}+ 2\alpha^k M) . \nonumber
\end{align}

Combining the two cases of (\ref{ge-dn-2}) and (\ref{ge-dn-4}), we can obtain
\begin{align}
\label{ge-dn-5}
    \E \eta^{k+1} \leq &  (1+\alpha^k L) \E \eta^{k} + \frac{2 \alpha^k M }{Z|\RR|} \\
   & +  4\rho |\RR| (\E \sqrt{H^k}+ 2\alpha^k M) . \nonumber
\end{align}

If we choose $\alpha^k = \frac{1}{L (k+k_0)}$, where $k$ is sufficiently large, we then derive
\begin{align}
\label{ge-dn-6}
    \E \eta^{k+1}
    \leq &     \lp 1+\frac{1}{k+k_0} \rp \E \eta^{k}
     +  \frac{2M}{L Z|\RR| (k+k_0)}   \\
     & +      \frac{4a_2 \rho |\RR| M}{L(k+k_0)}  .\nonumber
\end{align}

 Using telescopic cancellation on \eqref{ge-dn-6} from time $0$ to $k$, we deduce that

\begin{align}
\label{ge-dn-7}
    &\E \eta^{k}  \\
    \leq&   \sum_{k'=0}^{k-1} \lp \prod_{t=k'+1}^{k-1}(1+\frac{1}{t+k_0}) \rp \left[ \frac{2M}{L Z|\RR| (k'+k_0)}
    +
    % \lp(a_2)^{\frac{k}{2}} H^0
        \frac{4a_2 \rho |\RR| M}{L(k'+k_0)}   \right] \nonumber \\
        \overset{(a)}{\leq} &   \sum_{k'=0}^{k-1} \lp \prod_{t=k'+1}^{k-1} e^{ \frac{1}{t+k_0}} \rp \left[ \frac{2M}{L Z|\RR| (k'+k_0)}
    +
    % \lp(a_2)^{\frac{k}{2}} H^0
         \frac{4a_2 \rho |\RR| M}{L(k'+k_0)}   \right] \nonumber \\
        \leq &  \sum_{k'=0}^{k-1}  exp \lp  \sum_{t=k'+1}^{k-1} \frac{1}{t+k_0}\rp   \left[ \frac{2M}{L Z|\RR| (k'+k_0)}
    +
    % \lp(a_2)^{\frac{k}{2}} H^0
         \frac{4a_2 \rho |\RR| M}{L(k'+k_0)}   \right] \nonumber \\
       \overset{(b)}{\leq} &    \sum_{k'=0}^{k-1}  exp \lp  ln(\frac{k+k_0}{k'+k_0}) \rp    \left[ \frac{2M}{L Z|\RR| (k'+k_0)}
    +
    % \lp(a_2)^{\frac{k}{2}} H^0
        \frac{4a_2 \rho |\RR| M}{L(k'+k_0)}  \right] \nonumber \\
        \leq&  (k+k_0)  \sum_{k'=0}^{k-1}  (k'+k_0)^{-2}     \left[ \frac{2M}{L Z|\RR| }
    +
    % \lp(a_2)^{\frac{k}{2}} (k+k_0) H^0
         \frac{4a_2 \rho |\RR| M}{L}   \right] \nonumber  \\
                \leq&  (k+k_0) \left[ \frac{2M}{L Z|\RR| }
    +
    % \lp(a_2)^{\frac{k}{2}} (k+k_0) H^0
         \frac{4a_2 \rho |\RR| M}{L}   \right] ,\nonumber
       %  &\leq   \frac{2M}{\mu Z|\RR| } + \frac{2a_3(4\rho |\RR| +2\chi\sqrt{|\RR|}) M}{\mu} \nonumber \\
       %  &+ \frac{(4\rho |\RR| +2\chi\sqrt{|\RR|})H^0\sqrt{a_2}}{(\sqrt{a_2}-1)^2}\lp\sqrt{a_2}^K+\frac{1}{K+k_0-1}\rp  \nonumber \\
       % & + \frac{(4\rho |\RR| +2\chi\sqrt{|\RR|})H^0k_0\sqrt{a_2}}{(K+k_0-1)(1-\sqrt{a_2})}  \nonumber
\end{align}
where (a) uses $1+\vx \leq e^{\vx}$ and (b) holds true because
\begin{align}
\label{ge-dn-7-1}
&\sum_{t=k'+1}^{k-1} \frac{1}{t+k_0} =
  \sum_{t=k'+k_0}^{k+k_0-2} \frac{1}{t+1}
  % \leq \sum_{t=1}^{K+k_0-1} \frac{1}{t}
  \leq ln(\frac{k+k_0}{k'+k_0}).
\end{align}
Substituting (\ref{ge-d-16}) into (\ref{ge-dn-7}), according to Lemma \ref{lemma-sta}, we have
\begin{align}
\label{ge-dn-8}
  & \E_{\cS,\cL} [F(\bar\vx^k)-F_{\cS}(\bar\vx^k)] \\
    \leq  &\lp \frac{2M^2}{L Z|\RR| } + \frac{4a_2\rho |\RR|  M^2}{L}  \rp  (k+k_0), \nonumber
\end{align}
which completes the proof.

\subsection{Row Stochastic Virtual Mixing Matrix}
As previously mentioned, analyzing the recursion of $\eta^k$ is not suitable when the virtual mixing matrix is row stochastic. In this simulation, we directly analyze the recursion of $\delta^{k}$. We decompose $\delta^{k+1}$ into three parts as shown in (\ref{ge-1}) and analyze the first and third terms following the same approach as in (\ref{ge-2})--(\ref{ge-6}).

Now, we analyze the second term at the RHS of (\ref{ge-1}). For the non-Byzantine agent $i$, the probability of Algorithm \ref{robust-DSGD} selecting the same training sample from  $\cS_i$ and $\cS'_i$ at time $k$ is $1-\frac{1}{Z}$. By Lemma 6 in \cite{sun2021stability}, if Assumption \ref{assumption:Lip} holds, we have

\begin{align}
    \label{ge-n-1}
    & \|\bar\vx^{k+\frac{1}{2}} - \bar\vx'^{k+\frac{1}{2}} \| \\
     \leq  & (1+\alpha^k L)  \| \bar\vx^k - \bar\vx'^k\| + 2 \alpha^k L \E \sqrt{H^k}.  \nonumber
\end{align}
Substituting (\ref{ge-6}) and (\ref{ge-n-1}) into (\ref{ge-1}), in this case we have
\begin{align}
\label{ge-n-2}
     \E \delta^{k+1}
    \leq & (1+\alpha^k L)  \E \delta^{k}  + 2 \alpha^k L \E \sqrt{H^k} \\
  &+  (4\rho |\RR| +2\chi\sqrt{|\RR|})(\E \sqrt{H^k}+ 2\alpha^k M) .  \nonumber
\end{align}

For the non-Byzantine agent $i$, the probability of Algorithm \ref{robust-DSGD} selecting the different training samples from  $\cS_i$ and $\cS'_i$ at time $k$ is $\frac{1}{Z}$. By Lemma 6 in \cite{sun2021stability}, if Assumption \ref{assumption:Lip} holds, we have
\begin{align}
     \label{ge-n-3}
   &\|\bar\vx^{k+\frac{1}{2}} - \bar\vx'^{k+\frac{1}{2}} \|   \\
    % =& \frac{1}{|\RR|} \|\vx_i^{k+\frac{1}{2}}-\vx'^{k+\frac{1}{2}}_i +  \sum_{n \in \RR/\{i\}}(\vx^{k+\frac{1}{2}}_n-\vx'^{k+\frac{1}{2}}_n)\|  \nonumber \\
    % =& \frac{1}{|\RR|} \| \vx_i^k -\alpha^k \nabla f(\vx^{k}_n; \xi_{i,j})-\vx'^k + \alpha^k \nabla f(\vx^{k}_n; \xi'_{i,j}) \nonumber\\
    % &+  \sum_{n \in \RR/\{i\}}(\vx^{k+\frac{1}{2}}_n-\vx'^{k+\frac{1}{2}}_n) \|  \nonumber\\
    % \leq&  \frac{1}{|\RR|} \| \vx_i^k -\alpha^k \nabla f(\vx^{k}_n; \xi_{i,j})-\vx'^k + \alpha^k \nabla f(\vx^{k}_n; \xi_{i,j}) \nonumber \\
    % & +  \sum_{n \in \RR/\{i\}}(\vx^{k+\frac{1}{2}}_n-\vx'^{k+\frac{1}{2}}_n) \|  \nonumber \\
    % & + \frac{\alpha^k}{|\RR|}  \| \nabla f(\vx^{k}_n; \xi'_{i,j}) - \nabla f(\vx^{k}_n; \xi_{i,j})\| \nonumber\\
    \leq&  (1+\alpha^k L)\| \bar\vx^k - \bar\vx'^k\| + \frac{2 \alpha^k M}{|\RR|} + 2 \alpha^k L \E \sqrt{H^k} . \nonumber
\end{align}
Substituting (\ref{ge-6}) and (\ref{ge-n-3}) into (\ref{ge-1}), in this case we have
\begin{align}
\label{ge-n-4}
    \E \delta^{k+1} \leq &  (1+\alpha^k L) \E \delta^{k} + \frac{2 \alpha^k M }{|\RR|} + 2 \alpha^k L \E \sqrt{H^k}\\
    &+  (4\rho |\RR| +2\chi\sqrt{|\RR|})(\E \sqrt{H^k}+ 2\alpha^k M)  .\nonumber
\end{align}

Combining the two cases of (\ref{ge-n-2}) and (\ref{ge-n-4}), we can obtain
\begin{align}
\label{ge-n-5}
    \E \delta^{k+1} \leq &  (1+\alpha^k L) \E \delta^{k} + \frac{2 \alpha^k M }{Z|\RR|} + 2 \alpha^k L \E \sqrt{H^k} \\
    &+  (4\rho |\RR| +2\chi\sqrt{|\RR|})(\E \sqrt{H^k}+ 2\alpha^k M).  \nonumber
\end{align}

If we choose $\alpha^k = \frac{1}{L (k+k_0)}$, where $k_0$ is sufficiently large,  we then derive
\begin{align}
\label{ge-n-6}
    \E \delta^{k+1}
    \leq &  \lp 1+\frac{1}{ k+k_0} \rp  \E \delta^{k}  +  \frac{2M}{L Z|\RR| (k+k_0)}  \\
    &+\frac{2 \sqrt{c} M }{L (k+k_0)^2} +  (4\rho |\RR| +2\chi\sqrt{|\RR|})
    % \lp(a_2)^{\frac{k}{2}} H^0
         \frac{a_2 M}{L(k+k_0)}  . \nonumber
\end{align}
% where
% % $a_2 =1-\omega_1^2 \in (0,1) $ and
% $a_2 =\frac{\sqrt{2 \omega_2 a_1 }}{\omega_1} +1 $ is a constants.

 Using telescopic cancellation on \eqref{ge-n-6} from time $0$ to $k$, we deduce that

\begin{align}
\label{ge-n-7}
    \E \delta^{k}
    \leq&   \sum_{k'=0}^{k-1} \lp \prod_{t=k'+1}^{k-1}(1+\frac{1}{t+k_0}) \rp  \cdot\Bigg[ \frac{2M}{L Z|\RR| (k'+k_0)}  \\ &
    + \frac{2 \sqrt{c} M }{L (k'+k_0)^2} + (4\rho |\RR| +2\chi\sqrt{|\RR|})
        \frac{a_2 M}{L(k'+k_0)}   \Bigg] \nonumber \\
    %     \overset{(a)}{\leq} &   \sum_{k=1}^{K-1} \lp \prod_{t=k+1}^{K-1}exp \lp \frac{1}{t+k_0}\rp \rp   \cdot \Bigg[ \frac{2M}{L Z|\RR| (k+k_0)}  \nonumber\\ &
    % + \frac{2a_2 M }{L (k+k_0)^2} + (4\rho |\RR| +2\chi\sqrt{|\RR|})
    %     \frac{a_3 M}{L(k+k_0)}   \Bigg] \nonumber \\
    %     \leq &  \sum_{k=1}^{K-1}  exp \lp  \sum_{t=k+1}^{K-1} \frac{1}{t+k_0}\rp   \cdot\Bigg[ \frac{2M}{L Z|\RR| (k+k_0)}  \nonumber\\ &
    % + \frac{2a_2 M }{L (k+k_0)^2} + (4\rho |\RR| +2\chi\sqrt{|\RR|})
    %     \frac{a_3 M}{L(k+k_0)}   \Bigg] \nonumber \\
    %    \overset{(b)}{\leq} &    \sum_{k=1}^{K-1}  exp \lp   ln(\frac{K+k_0}{k+k_0}) \rp    \cdot\Bigg[ \frac{2M}{L Z|\RR| (k+k_0)}  \nonumber\\ &
    % + \frac{2a_2 M }{L (k+k_0)^2} + (4\rho |\RR| +2\chi\sqrt{|\RR|})
    %     \frac{a_3 M}{L(k+k_0)}   \Bigg] \nonumber \\
        \leq&  (k+k_0)  \sum_{k'=0}^{k-1}  (k'+k_0)^{-2}     \cdot\Bigg[ \frac{2M}{L Z|\RR| }  \nonumber\\ &
    + \frac{2\sqrt{c} M }{L (k'+k_0)} + (4\rho |\RR| +2\chi\sqrt{|\RR|})
        \frac{a_2 M}{L}   \Bigg] \nonumber \\
        \leq&  (k+k_0)   \Bigg[ \frac{2M}{L Z|\RR| }
    + \frac{2 \sqrt{c} M }{L} + (4\rho |\RR| +2\chi\sqrt{|\RR|})
        \frac{a_2 M}{L}   \Bigg]. \nonumber
       %  &\leq   \frac{2M}{\mu Z|\RR| } + \frac{2a_3(4\rho |\RR| +2\chi\sqrt{|\RR|}) M}{\mu} \nonumber \\
       %  &+ \frac{(4\rho |\RR| +2\chi\sqrt{|\RR|})H^0\sqrt{a_2}}{(\sqrt{a_2}-1)^2}\lp\sqrt{a_2}^K+\frac{1}{K+k_0-1}\rp  \nonumber \\
       % & + \frac{(4\rho |\RR| +2\chi\sqrt{|\RR|})H^0k_0\sqrt{a_2}}{(K+k_0-1)(1-\sqrt{a_2})}  \nonumber
\end{align}
% where (a) uses $1+\vx \leq e^{\vx}$ and (b) holds true because
% \begin{align}
% \label{ge-n-8}
% &\sum_{t=k+1}^{K-1} \frac{1}{t+k_0} =
%   \sum_{t=k+k_0}^{K+k_0-2} \frac{1}{t+1}
%   % \leq \sum_{t=1}^{K+k_0-1} \frac{1}{t}
%   \leq ln(\frac{K+k_0}{k+k_0})
% \end{align}

% According to

% \begin{align}
% \label{ge-n-9}
% &\sum_{k=1}^{K-1}  (k+k_0)^{-\frac{cL}{\mu}-1}  =  \sum_{t=k_0}^{K+k_0-2}  (t+1)^{-\frac{cL}{\mu}-1}   \\
%  \leq & \sum_{t=k_0}^{K+k_0-2} \int_{t}^{t+1} x^{-\frac{cL}{\mu}-1} dx \leq \int_{k_0}^{K+k_0} x^{-\frac{cL}{\mu}-1} dx  \nonumber \\
%  \leq & \frac{\mu}{cL} \lp k_0^{-\frac{cL}{\mu}} - (K+k_0)^{-\frac{cL}{\mu}}\rp  \nonumber
% \end{align}

% Combining (\ref{ge-n-8}) and (\ref{ge-n-9}), we can derive

% \begin{align}
% \label{ge-n-10}
%     &\E \delta^{k} \\
%         \leq&  \frac{\mu}{cL}  (K+k_0)^{\frac{cL}{\mu}}  k_0^{-\frac{cL}{\mu}}  \lp \frac{2cM}{\mu Z|\RR| } + \frac{2a_2c (4\rho |\RR| +2\chi\sqrt{|\RR|})M}{\mu}\rp
%         \nonumber
%     %     \\ &\left[ \frac{2M}{\mu Z|\RR| }
%     % +  (4\rho |\RR| +2\chi\sqrt{|\RR|}) \lp(a_2)^{\frac{k}{2}} (k+k_0) H^0
%         % + \frac{2a_3 M}{\mu} \rp  \right] \nonumber
%        %  &\leq   \frac{2M}{\mu Z|\RR| } + \frac{2a_3(4\rho |\RR| +2\chi\sqrt{|\RR|}) M}{\mu} \nonumber \\
%        %  &+ \frac{(4\rho |\RR| +2\chi\sqrt{|\RR|})H^0\sqrt{a_2}}{(\sqrt{a_2}-1)^2}\lp\sqrt{a_2}^K+\frac{1}{K+k_0-1}\rp  \nonumber \\
%        % & + \frac{(4\rho |\RR| +2\chi\sqrt{|\RR|})H^0k_0\sqrt{a_2}}{(K+k_0-1)(1-\sqrt{a_2})}  \nonumber
% \end{align}

Substituting (\ref{ge-16}) into (\ref{ge-n-7}), according to Lemma \ref{lemma-sta}, we have
\begin{align}
    \label{ge-n-111}
   &  \E_{\cS,\mathcal{L}}[F(\mathcal{L}(\cS))-F_{\cS}(\mathcal{L}(\cS))]  \\
    \leq&     (k+k_0)   \Bigg[ \frac{2M^2}{L Z|\RR| }
    + \frac{2\sqrt{c} M^2 }{L} + (4\rho |\RR| +2\chi\sqrt{|\RR|})
        \frac{a_2 M^2}{L}   \Bigg], \nonumber
    % &+ (4\rho |\RR| +2\chi\sqrt{|\RR|}) \frac{\sqrt{a_2}}{1-\sqrt{a_2}} H^0 M
\end{align}
which completes the proof.
\end{proof}

\section{Improved generalization error analysis with non-convex loss}
\label{app-im}
Following \cite{bars2023improved} and \cite{hardt2016train}, we can improve the generalization error analysis in Theorem \ref{thm-ge-n} by considering the first time when the different training samples from $\cS$ and $\cS'$ are selected. We further assume that the non-convex loss $f(\vx; \xi)$ is within $[0,1]$, without which the established generalization error will depend on the supreme of $f(\vx; \xi)$.

\begin{theorem}[Generalization error of Byzantine-resilient DSGD with non-convex loss]
\label{thm-ge-in}
Suppose that the robust aggregation rules $\{\A_n\}_{n\in \RR}$ in Algorithm \ref{robust-DSGD} satisfy Definition \ref{definition:mixing-matrix}, the associated contraction constant satisfies $\rho < \rho^* := \frac{\beta}{8\sqrt{\vert \RR \vert}}$, and the loss $f(\vx; \xi)$ is non-convex and within $[0,1]$. Set the step size $\alpha^k= \frac{a}{L (k+k_0)}$, where $k_0$ is sufficiently large and $a>0$.
Under Assumptions \ref{assumption:connection}--\ref{assumption:indSampling}, at any given time $k$, the generalization error of Algorithm \ref{robust-DSGD} is bounded by
    \begin{align}
        \label{thm-ge-in-1}
        & \E_{\cS,\mathcal{L}}[F(\bar \vx^k)-F_{\cS}(\bar \vx^k)] \leq  \frac{c_2 1_{\chi \neq 0} M^2   }{L \Delta} \\
         &  + (\frac{1}{  Z}+\frac{1}{aZ}) \Delta^{\frac{1}{a+1}}   (k+k_0)^{\frac{a}{a+1}}. \nonumber
        % &+ c_6 (4\rho |\RR| +2\chi\sqrt{|\RR|})H^0 M  .  \nonumber
       % & + \frac{c_6 (4\rho |\RR| +2\chi\sqrt{|\RR|})H^0Mk_0}{k+k_0-1}  \nonumber
    \end{align}
     Here $c_1, c_2>0$ are constants and $$\Delta = \lp \frac{2aM^2}{L |\RR| } +  \frac{ 4ac_1\rho |\RR| Z M^2}{L} +  \frac{ 2a c_1\chi\sqrt{|\RR|} Z M^2}{L} \rp  . $$
\end{theorem}

Again, in the case that $\rho=\chi=0$, the derived generalization error bound in Theorem \ref{thm-ge-n} is in the order of $O(\frac{ k^{\frac{a}{a+1}} }{ Z|\RR|^{\frac{1}{a+1}} })$.  Such a generalization error bound matches the ones obtained for DSGD \cite{bars2023improved}. It is notably tighter than those established in \cite{sun2021stability} and \cite{deng2023stability}, in which the bounds contain additional terms that do not vanish when $Z$ and $|\RR|$ go to infinity.
However, it does not precisely align with the generalization error bound obtained for SGD, which is in the order of $O(\frac{ k^{\frac{a}{a+1}} }{ Z|\RR| })$. This discrepancy arises because the probability of selecting the different training samples in each iteration is $\frac{1}{Z|\RR|}$ in SGD, whereas becomes $\frac{1}{|Z|}$ in DSGD, as discussed in \cite{bars2023improved}.

When the Byzantine agents are present such that the contraction constant $\rho \neq 0$ in general, the induced generalization error term is in the order of $O((\frac{ \rho |\RR| M^2 }{L})^{\frac{1}{a+1}} k^{\frac{a}{a+1}})$. Furthermore, a non-doubly stochastic virtual mixing matrix $W$, which implies $\chi$ $\neq 0$, yields two additional error terms in the orders of $O(\frac{M^2}{L})$ and $O((\frac{ \chi \sqrt{|\RR|} M^2 }{L})^{\frac{1}{a+1}} k^{\frac{a}{a+1}})$, respectively.

\subsection{Doubly Stochastic Virtual Mixing Matrix}
\begin{proof}
    According to Lemma 3.11 in \cite{hardt2016train}, for every $\hat k \in \{0,1,...,k \}$, we have
    \begin{align}
        \label{ge-idn-1}
        \E[ f(\bar\vx^k,\xi) - f(\bar\vx'^k,\xi)] \leq \frac{\hat k}{Z} + M \E [\eta^k | \eta^{\hat k}=0].
    \end{align}

In the same way as we have proved (\ref{ge-dn-6}), if we choose $\alpha^k =$ $\frac{a}{L (k+k_0)}$, for all $k \geq \hat k$ we have
% \begin{align}
% \label{ge-idn-2}
%     \E [\eta^{k+1} | \eta^{k'}=0] \leq &  (1+\alpha^k L) \E \eta^{k} + \frac{2 \alpha^k M }{Z|\RR|} \\
%    & +  4\rho |\RR| (\E \sqrt{H^k}+ 2\alpha^k M)  \nonumber
% \end{align}
% If we choose $\alpha^k = \frac{1}{L (k+k_0)}$, we then derive
\begin{align}
\label{ge-idn-2}
    \E [\eta^{k+1} | \eta^{\hat k}=0]
    \leq &     \lp 1+\frac{a}{k+k_0} \rp \E [\eta^{k} | \eta^{\hat k}=0] \\
     & +  \frac{2aM}{L Z|\RR| (k+k_0)}
      +      \frac{4 a a_2 \rho |\RR| M}{L(k+k_0)} . \nonumber
\end{align}
Since $\eta^{\hat k}=0$, using telescopic cancellation on \eqref{ge-idn-2} from time $\hat k$ to $k$, we deduce that
\begin{align}
\label{ge-idn-3}
    &\E [\eta^{k} | \eta^{\hat k}=0]  \\
    % \leq&   \sum_{k'=\hat k}^{k-1} \lp \prod_{t=k'+1}^{k-1}(1+\frac{a}{t+k_0}) \rp \left[ \frac{2aM}{L Z|\RR| (k'+k_0)}
    % +
    % % \lp(a_2)^{\frac{k}{2}} H^0
    %     \frac{4aa_2 \rho |\RR| M}{L(k'+k_0)}   \right] \nonumber \\
        \leq &   \sum_{k'=\hat k}^{k-1} \lp \prod_{t=k'+1}^{k-1} e^{ \frac{a}{t+k_0}} \rp \left[ \frac{2aM}{L Z|\RR| (k'+k_0)}
    +
    % \lp(a_2)^{\frac{k}{2}} H^0
         \frac{4aa_2 \rho |\RR| M}{L(k'+k_0)}   \right] \nonumber \\
        \leq &  \sum_{k'=\hat k}^{k-1}  exp \lp  \sum_{t=k'+1}^{k-1} \frac{a}{t+k_0}\rp   \left[ \frac{2aM}{L Z|\RR| (k'+k_0)}
    +
    % \lp(a_2)^{\frac{k}{2}} H^0
         \frac{4aa_2 \rho |\RR| M}{L(k'+k_0)}   \right] \nonumber \\
       \leq &    \sum_{k'=\hat k}^{k-1}  exp \lp a ln(\frac{k+k_0}{k'+k_0}) \rp    \left[ \frac{2aM}{L Z|\RR| (k'+k_0)}
    +
    % \lp(a_2)^{\frac{k}{2}} H^0
        \frac{4aa_2 \rho |\RR| M}{L(k'+k_0)}  \right] \nonumber \\
        \leq&  (k+k_0)^a  \sum_{k'=\hat k}^{k-1}  (k'+k_0)^{-a-1}     \left[ \frac{2aM}{L Z|\RR| }
    +
    % \lp(a_2)^{\frac{k}{2}} (k+k_0) H^0
         \frac{4aa_2 \rho |\RR| M}{L}   \right] \nonumber  \\
                \leq&  \lp \frac{k+k_0}{\hat k} \rp ^a \left[ \frac{2M}{L Z|\RR| }
    +
    % \lp(a_2)^{\frac{k}{2}} (k+k_0) H^0
         \frac{4 a_2 \rho |\RR| M}{L}   \right] ,\nonumber
       %  &\leq   \frac{2M}{\mu Z|\RR| } + \frac{2a_3(4\rho |\RR| +2\chi\sqrt{|\RR|}) M}{\mu} \nonumber \\
       %  &+ \frac{(4\rho |\RR| +2\chi\sqrt{|\RR|})H^0\sqrt{a_2}}{(\sqrt{a_2}-1)^2}\lp\sqrt{a_2}^K+\frac{1}{K+k_0-1}\rp  \nonumber \\
       % & + \frac{(4\rho |\RR| +2\chi\sqrt{|\RR|})H^0k_0\sqrt{a_2}}{(K+k_0-1)(1-\sqrt{a_2})}  \nonumber
\end{align}
where the last inequality holds true because
\begin{align}
    \label{ge-idn-4}
    \sum_{k'=\hat k}^{k-1}  (k'+k_0)^{-a-1} \leq & \int_{\hat k}^{k+k_0} x^{-a-1} dx \\
    = & \frac{1}{a}(\hat k^{-a}-(k+k_0)^{-a}) . \nonumber
\end{align}
Substituting (\ref{ge-idn-3}) into (\ref{ge-idn-1}), we have
    \begin{align}
        \label{ge-idn-5}
       & \E[ f(\bar\vx^k,\xi) - f(\bar\vx'^k,\xi)] \\
        \leq & \frac{\hat k}{Z} +  \lp \frac{k+k_0}{\hat k} \rp ^a \left[ \frac{2M^2}{L Z|\RR| }.
    +
    % \lp(a_2)^{\frac{k}{2}} (k+k_0) H^0
         \frac{4 a_2 \rho |\RR| M^2}{L}.   \right] \nonumber
    \end{align}

Setting $\hat k= ( \frac{2 a M^2}{L |\RR|  } +
         \frac{4 a a_2 \rho |\RR| Z M^2}{L}   )^{\frac{1}{a+1}}   (k+k_0)^{\frac{a}{a+1}}$, according to Lemma \ref{lemma-sta}, we have
         \begin{align}
        \label{ge-idn-6}
       &  \E_{\cS,\mathcal{L}}[F(\bar \vx^k)-F_{\cS}(\bar \vx^k)]  \\
        \leq & (\frac{1}{  Z}+\frac{1}{aZ})  \lp  \frac{2 a M^2}{L |\RR|  } +
         \frac{4 a a_2 \rho |\RR| Z M^2}{L}   \rp^{\frac{1}{a+1}} (k+k_0)^{\frac{a}{a+1}},\nonumber
    \end{align}
which completes the proof.

\subsection{Row Stochastic Virtual Mixing Matrix}
According to Lemma 3.11 in \cite{hardt2016train}, for every $\hat k \in \{0,1,...,$ $k \}$, we have
    \begin{align}
        \label{ge-irn-1}
        \E[ f(\bar\vx^k,\xi) - f(\bar\vx'^k,\xi)] \leq \frac{\hat k}{Z} + M \E [ \delta^k | \delta^{\hat k}=0].
    \end{align}

In the same way as we have proved (\ref{ge-n-6}), if we choose $\alpha^k =$ $\frac{a}{L (k+k_0)}$, for all $k \geq \hat k$ we have
\begin{align}
\label{ge-irn-2}
    & \E [\delta^{k+1} | \delta^{\hat k}=0]
    \leq    \lp 1+\frac{a}{ k+k_0} \rp  \E [\delta^k | \delta^{\hat k}=0] \\ &+  \frac{2aM}{L Z|\RR| (k+k_0)}
    +\frac{2 a \sqrt{c} M }{L (k+k_0)^2} \nonumber \\ &+  (4\rho |\RR| +2\chi\sqrt{|\RR|})
    % \lp(a_2)^{\frac{k}{2}} H^0
         \frac{a a_2 M}{L(k+k_0)} .  \nonumber
\end{align}
Since $\delta^{\hat k}=0$, using telescopic cancellation on \eqref{ge-irn-2} from time $\hat k$ to $k$, we deduce that
\begin{align}
\label{ge-irn-3}
    &\E [\delta^k | \delta^{\hat k}=0]  \\
    \leq&   \sum_{k'=\hat k}^{k-1} \lp \prod_{t=k'+1}^{k-1}(1+\frac{a}{t+k_0}) \rp  \cdot\Bigg[ \frac{2aM}{L Z|\RR| (k'+k_0)}  \nonumber\\ &
    + \frac{2 a \sqrt{c} M }{L (k'+k_0)^2} + (4\rho |\RR| +2\chi\sqrt{|\RR|})
        \frac{a a_2 M}{L(k'+k_0)}   \Bigg] \nonumber \\
    %     \overset{(a)}{\leq} &   \sum_{k=1}^{K-1} \lp \prod_{t=k+1}^{K-1}exp \lp \frac{1}{t+k_0}\rp \rp   \cdot \Bigg[ \frac{2M}{L Z|\RR| (k+k_0)}  \nonumber\\ &
    % + \frac{2a_2 M }{L (k+k_0)^2} + (4\rho |\RR| +2\chi\sqrt{|\RR|})
    %     \frac{a_3 M}{L(k+k_0)}   \Bigg] \nonumber \\
    %     \leq &  \sum_{k=1}^{K-1}  exp \lp  \sum_{t=k+1}^{K-1} \frac{1}{t+k_0}\rp   \cdot\Bigg[ \frac{2M}{L Z|\RR| (k+k_0)}  \nonumber\\ &
    % + \frac{2a_2 M }{L (k+k_0)^2} + (4\rho |\RR| +2\chi\sqrt{|\RR|})
    %     \frac{a_3 M}{L(k+k_0)}   \Bigg] \nonumber \\
    %    \overset{(b)}{\leq} &    \sum_{k=1}^{K-1}  exp \lp   ln(\frac{K+k_0}{k+k_0}) \rp    \cdot\Bigg[ \frac{2M}{L Z|\RR| (k+k_0)}  \nonumber\\ &
    % + \frac{2a_2 M }{L (k+k_0)^2} + (4\rho |\RR| +2\chi\sqrt{|\RR|})
    %     \frac{a_3 M}{L(k+k_0)}   \Bigg] \nonumber \\
        \leq&  (k+k_0)^a  \sum_{k'= \hat k}^{k-1}  (k'+k_0)^{-a-2}
         \frac{2a\sqrt{c} M }{L} + (k+k_0)^a \cdot   \nonumber \\
         &   \sum_{k'= \hat k}^{k-1}  (k'+k_0)^{-a-1} \Bigg[ \frac{2aM}{L Z|\RR| }
     + (4\rho |\RR| +2\chi\sqrt{|\RR|})
        \frac{aa_2 M}{L}   \Bigg] \nonumber \\
        \leq&  \lp \frac{k+k_0}{\hat k} \rp ^a \left[ \frac{2M}{L Z|\RR| } +
         (4\rho |\RR| +2\chi\sqrt{|\RR|})
        \frac{a_2 M}{L}    \right]  \nonumber \\  &+ \frac{2\sqrt{c} M(k+k_0)^a}{L \hat k^{a+1}} . \nonumber
       %  &\leq   \frac{2M}{\mu Z|\RR| } + \frac{2a_3(4\rho |\RR| +2\chi\sqrt{|\RR|}) M}{\mu} \nonumber \\
       %  &+ \frac{(4\rho |\RR| +2\chi\sqrt{|\RR|})H^0\sqrt{a_2}}{(\sqrt{a_2}-1)^2}\lp\sqrt{a_2}^K+\frac{1}{K+k_0-1}\rp  \nonumber \\
       % & + \frac{(4\rho |\RR| +2\chi\sqrt{|\RR|})H^0k_0\sqrt{a_2}}{(K+k_0-1)(1-\sqrt{a_2})}  \nonumber
\end{align}
Substituting (\ref{ge-irn-3}) into (\ref{ge-irn-1}), we have
\begin{align}
        \label{ge-irn-4}
       & \E[ f(\bar\vx^k,\xi) - f(\bar\vx'^k,\xi)]
        \leq  \frac{\hat k}{Z} + \frac{2\sqrt{c} M^2 (k+k_0)^a}{L \hat k^{a+1}}  \\ & +\lp \frac{k+k_0}{\hat k} \rp ^a \left[ \frac{2M^2}{L Z|\RR| }
   + (4\rho |\RR| +2\chi\sqrt{|\RR|})
        \frac{a_2 M^2}{L}    \right] .\nonumber
    \end{align}

Now, setting $\hat k= ( \frac{2 a M^2}{L |\RR|  } +
         \frac{ (4\rho |\RR| +2\chi\sqrt{|\RR|})a a_2 Z M^2}{L}   )^{\frac{1}{a+1}}   (k+k_0)^{\frac{a}{a+1}}$, according to Lemma \ref{lemma-sta},we have
         \begin{align}
        \label{ge-irn-5}
       & | \E_{\cS,\mathcal{L}}[F(\mathcal{L}(\cS))-F_{\cS}(\mathcal{L}(\cS))] | \\
        \leq  &  \frac{2\sqrt{c} M^2 }{L} \lp \frac{2 a M^2}{L |\RR|  } +
         (4\rho |\RR| +2\chi\sqrt{|\RR|})
        \frac{a a_2 Z M^2}{L}    \rp^{-1} + \nonumber \\
        & (\frac{1}{  Z}+\frac{1}{aZ})  \lp \frac{2 a M^2}{L |\RR|  } +
         (4\rho |\RR| +2\chi\sqrt{|\RR|})
        \frac{a a_2 Z M^2}{L}    \rp^{\frac{1}{a+1}} \cdot \nonumber \\
        & (k+k_0)^{\frac{a}{a+1}}, \nonumber
    \end{align}
which completes the proof.
\end{proof}

\section{Useful Lemma and its proof}
\label{proof-a}
The following lemma characterizes how the non-Byzantine agents reach consensus during the learning process. The proof follows that of Theorem 1 in \cite{Ye2023},
%For completeness, we provide a detailed proof here. The key difference is that while \cite{Ye2023} utilizes the clipped stochastic gradients with added noise, this paper employs the standard stochastic gradients.
whereas the latter involves stochastic gradient clipping and added noise. Define the disagreement measure as
    \begin{align}
        \label{hk}
		H^k =  \frac{1}{\vert \RR \vert}\sum_{n\in \RR} \|\vx^{k}_{n}-\bar\vx^{k}\|^2,
	\end{align}
where $\bar\vx^k:= \frac{1}{|\RR|}\sum_{n\in \RR}\vx^k_n$ is the average of all local models of the non-Byzantine agents at time $k$.
\begin{lemma}[Consensus of Byzantine-resilient DSGD]
\label{lemma-dm}
    Suppose that the robust aggregation rules $\{\A_n\}_{n\in \RR}$
    in Algorithm \ref{robust-DSGD} satisfy \eqref{inequality:robustness-of-aggregation-local}, and
      $\rho < \rho^* := \frac{\beta}{8\sqrt{\vert \RR \vert}}$.
    Set the step size $\alpha^k= \frac{1}{k+k_0}$, where $k_0$ is sufficiently large. When all non-Byzantine agents have same the initialization, under Assumptions \ref{assumption:connection} and \ref{assumption:gradients}, for Algorithm \ref{robust-DSGD}, we have
    \begin{align}
        \label{inequality:H-convergence}
        % \E H^k
        % \le & (c_1)^k H^0 + \frac{c_2 M^2+c_3 C^2}{\mu^2} \frac{1}{(k+k_0)^2}.
        \E H^k
        \le &  \frac{cM^2}{(k+k_0)^2}.
    \end{align}
    Here  $c > 0$ is a constant, and the expectation is taken over all the random variables.
\end{lemma}
\begin{proof}
      For convenience, define $X^{k} = [\vx^{k}_1, \cdots, \vx^{k}_{|\RR|}]^\top \in$ $\mathbb{R}^{|\RR| \times D}$ and  $X^{k+\frac{1}{2}} = [\vx^{k+\frac{1}{2}}_1, \cdots, \vx^{k+\frac{1}{2}}_{|\RR|}]^\top \! \in \mathbb{R}^{|\RR| \times D}$.
      For any $u \in (0, 1)$, it holds
    \begin{align}
        \label{dm-1}
         &\sum_{n\in \RR}   \|\vx^{k+1}_{n}-\bar\vx^{k+1}\|^2   \\
        = & \|(I-\frac{1}{\vert \RR \vert}\bm{1}\bm{1}^\top) X^{k+1}\|_{F}^2  \nonumber\\
%        = \|(I-\frac{1}{\vert \RR \vert}\bm{1}\bm{1}^\top)  A^{k+\frac{1}{2}}\|_{2,\infty}^2
        % ==========================
         \le & \frac{1}{1-u} \|(I-\frac{1}{\vert \RR \vert}\bm{1}\bm{1}^\top)W X^{k+\frac{1}{2}}\|_{F}^2
        + \frac{2}{u} \|X^{k+1}-WX^{k+\frac{1}{2}}\|_{F}^2
        \nonumber \\
        &
        + \frac{2}{u} \|\frac{1}{\vert \RR \vert}\bm{1}\bm{1}^\top X^{k+1}-\frac{1}{\vert \RR \vert}\bm{1}\bm{1}^\top WX^{k+\frac{1}{2}}\|_{F}^2,  \nonumber
    \end{align}
     where the equality is due to $\bar\vx^{k+1} = \frac{1}{|\RR|}\sum_{n\in \RR}\vx^{k+1}_n$ and the inequality comes from $\|\bm{a}+\bm{b}+\bm{c}\|^2\le \frac{1}{1-u}\|\bm{a}\|^2+\frac{2}{u}\|\bm{b}\|^2+\frac{2}{u}\|\bm{c}\|^2$.

     For the first-term at the RHS of \eqref{dm-1}, we have
    \begin{align}
        \label{dm-2}
        & \|(I-\frac{1}{\vert \RR \vert}\bm{1}\bm{1}^\top) WX^{k+\frac{1}{2}}\|_{F}^2
         \\
        % ==========================
        =&  \|(I-\frac{1}{\vert \RR \vert}\bm{1}\bm{1}^\top)W (I-\frac{1}{\vert \RR \vert}\bm{1}\bm{1}^\top) X^{k+\frac{1}{2}}\|_{F}^2
        \nonumber \\
        % ==========================
        \le&  \|(I-\frac{1}{\vert \RR \vert}\bm{1}\bm{1}^\top) W\|^2\|(I-\frac{1}{\vert \RR \vert}\bm{1}\bm{1}^\top) X^{k+\frac{1}{2}}\|_{F}^2
        \nonumber \\
        % ==========================
        =& (1-\beta)  \| (I-\frac{1}{\vert \RR \vert}\bm{1}\bm{1}^\top) X^{k+\frac{1}{2}}\|_{F}^2.  \nonumber
    \end{align}
      For the second term at the RHS of \eqref{dm-1}, we use the contraction property of  $\{ \A_n \}_{n \in  \RR}$ in \eqref{inequality:robustness-of-aggregation-local} to derive
    \begin{align}
        \label{dm-3}
        & \|X^{k+1}-WX^{k+\frac{1}{2}}\|_{F}^2
          \\
        % ==========================
        =&  \sum_{n\in \RR}  \|\A_n (\vx^{k+\frac{1}{2}}_n, \{\vx^{k+\frac{1}{2}}_{m}\}_{m\in \N_n})-\hat\vx^{k+\frac{1}{2}}_n\|^2
        \nonumber \\
        % ==========================
        \le& \rho^2\sum_{n\in \RR} \max_{m\in \RR_n \cup \{n\}}\| \vx^{k+\frac{1}{2}}_{m}-\hat\vx^{k+\frac{1}{2}}_n\|^2
        \nonumber \\
        % ==========================
        \le& 4\rho^2 \sum_{n\in \RR}  \max_{m\in \RR}\| \vx^{k+\frac{1}{2}}_{m}-\bar\vx^{k+\frac{1}{2}}\|^2  \nonumber \\
        \le& 4\rho^2 \vert \RR \vert \sum_{n\in \RR}  \| \vx^{k+\frac{1}{2}}_{n}-\bar\vx^{k+\frac{1}{2}}\|^2 \nonumber \\
        = & 4\rho^2 \vert \RR \vert  \|(I-\frac{1}{\vert \RR \vert}\bm{1}\bm{1}^\top) X^{k+\frac{1}{2}}\|_{F}^2,  \nonumber
    \end{align}
    where $\hat\vx^{k+\frac{1}{2}}_n = \sum_{m \in \RR_n\cup\{n\}}w_{nm} \vx^{k}_{m,n}$ is the weighted average and $\bar\vx^{k+\frac{1}{2}} =  \frac{1}{|\RR|}\sum_{n\in \RR}\vx_{n}^{k+\frac{1}{2}}$ is the average. The second inequality holds true because

    \begin{align}
    \label{dm-4}
		& \max_{m\in  \RR_n \cup \{n\}}  \|\vx^{k+\frac{1}{2}}_{m}-\hat\vx^{k+\frac{1}{2}}_n\|^2
		  \\
		% ==========================
		\le& 2  \max_{m\in  \RR_n \cup \{n\}}  \|\vx^{k+\frac{1}{2}}_{m}-\bar\vx^{k+\frac{1}{2}}\|^2
		+2  \|\bar\vx^{k+\frac{1}{2}}- \hat\vx^{k+\frac{1}{2}}_n\|^2
		\nonumber \\
		% ==========================
		\le& 2   \max_{m\in  \RR_n \cup \{n\}}  \|\vx^{k+\frac{1}{2}}_{m}-\bar\vx^{k+\frac{1}{2}}\|^2
		+2 \max_{n\in\RR}\|\vx^{k+\frac{1}{2}}_{n}-\bar\vx^{k+\frac{1}{2}}\|^2 \nonumber \\
            \le& 4 \max_{n\in\RR}\|\vx^{k+\frac{1}{2}}_{n}-\bar\vx^{k+\frac{1}{2}}\|^2  .\nonumber
	\end{align}
     The similar technique can be applied to the third term at the RHS of \eqref{dm-1}, yielding
    \begin{align}
        \label{dm-5}
        & \|\frac{1}{\vert \RR \vert}\bm{1}\bm{1}^\top X^{k+1}-\frac{1}{\vert \RR \vert}\bm{1}\bm{1}^\top WX^{k+\frac{1}{2}}\|_{F}^2 \\
        \leq &   \|\frac{1}{\vert \RR \vert}\bm{1}\bm{1}^\top\|^2 \| X^{k+1} - WX^{k+\frac{1}{2}}\|_{F}^2   \nonumber \\
        \leq & 4\rho^2 |\RR|   \|(I-\frac{1}{\vert \RR \vert}\bm{1}\bm{1}^\top) X^{k+\frac{1}{2}}\|_{F}^2  . \nonumber
    \end{align}

Substituting (\ref{dm-2}), (\ref{dm-3}) and (\ref{dm-5}) into (\ref{dm-1}) and taking expectation over all random variables up to time $k+1$, we have
\begin{align}
    \label{dm-6}
     &  \sum_{n\in \RR} \E \|\vx^{k+1}_{n}-\bar\vx^{k+1}\|^2   \\
      &\le \lp \frac{1-\beta}{1-u}
        + \frac{16\rho^2\vert \RR \vert}{u} \rp \E \|(I-\frac{1}{\vert \RR \vert}\bm{1}\bm{1}^\top)X^{k+\frac{1}{2}}\|_{F}^2.   \nonumber
\end{align}
For (\ref{dm-6}), it holds for any $\gamma \in (0,1)$ that
    \begin{align}
        \label{dm-7}
        &   \E\|(I-\frac{1}{\vert \RR \vert}\bm{1}\bm{1}^\top) X^{k+\frac{1}{2}}\|_{F}^2   \\
         = & \sum_{n\in  \RR}\E \| \vx^{k+\frac{1}{2}}_{n}-\bar\vx^{k+\frac{1}{2}}\|^2  \nonumber \\
    \leq & \frac{1}{1-\gamma} \sum_{n\in\RR} \E \|  \vx^{k}_{n}-\bar\vx^{k} \|^2 \nonumber \\
    % & + \frac{2(\alpha^k)^2}{\gamma} \sum_{n\in\RR} \E \| \nabla f(\vx^{k}_n; \xi_n^{k}) - \frac{1}{\vert \RR \vert} \sum_{m\in\RR} \nabla f(\vx^{k}_m; \xi_m^{k})  \|^2 \nonumber \\
    % &+ \frac{2(\alpha^k)^2(\lambda^k)^2}{\gamma} \sum_{n\in\RR} \E \| sign(\nabla f(\vx^{k}_n; \xi_n^{k})) - \frac{1}{\vert \RR \vert} \sum_{m\in\RR} sign(\nabla f(\vx^{k}_m; \xi_m^{k}))  \|^2 \nonumber \\
    &+ \frac{(\alpha^k)^2}{\gamma} \sum_{n\in\RR} \E \| \nabla f(\vx^{k}_n; \xi_n^{k}) - \frac{1}{\vert \RR \vert} \sum_{m\in\RR}  \nabla f(\vx^{k}_m; \xi_m^{k})  \|^2 \nonumber \\
     \leq & \frac{1}{1-\gamma} \sum_{n\in\RR} \E \|  \vx^{k}_{n}-\bar\vx^{k} \|^2 + \frac{4(\alpha^k)^2|\RR|M^2}{\gamma}  .\nonumber
     % & + \frac{2(\alpha^k)^2}{\gamma} \sum_{n\in\RR} \E \| \frac{1}{\vert \RR \vert} \sum_{m\in\RR} \lp \nabla f(\vx^{k}_m; \xi_m^{k}) - \lambda^k sign(\nabla f(\vx^{k}_m; \xi_m^{k})) \rp \|^2 \nonumber
    % \leq & \frac{1}{1-\gamma} \sum_{n\in\RR} \E \|  \vx^{k}_{n}-\bar\vx^{k} \|^2 +  \frac{8(\alpha^k)^2|\RR|(M^2+(\lambda^k)^2d)}{\gamma} \nonumber
\end{align}
Applying (\ref{dm-7}) to (\ref{dm-6}), we have
    \begin{align}
        \label{dm-8}
         & \E H^{k+1}
        = \frac{1}{\vert \RR \vert} \E\|(I-\frac{1}{\vert \RR \vert}\bm{1}\bm{1}^\top) X^{k+1}\|_{F}^2  \\
        \leq & \frac{1}{\vert \RR \vert}  \lp \frac{1-\beta}{1-u}
        + \frac{16\rho^2\vert \RR \vert}{u} \rp \E \|(I-\frac{1}{\vert \RR \vert}\bm{1}\bm{1}^\top)X^{k+\frac{1}{2}}\|_{F}^2 \nonumber \\
        \leq &   \lp \frac{1-\beta}{1-u}
        + \frac{16\rho^2\vert \RR \vert}{u} \rp \lp \frac{1}{1-\gamma} \E H^k
        +\frac{4(\alpha^k)^2M^2}{\gamma}\rp  .\nonumber
        % \leq &   \lp 1-\beta
        % + 8\rho \sqrt{\vert \RR \vert} \rp \lp \frac{1}{1-\gamma} \E H^k
        % +\frac{8(\alpha^k)^2(M^2+(\lambda^k)^2d)}{\gamma}\rp  \nonumber \\
        % \leq &   \lp 1- \omega_1 \rp \lp \frac{1}{1-\gamma} \E H^k
        % +\frac{8(\alpha^k)^2(M^2+(\lambda^k)^2d)}{\gamma}\rp   \nonumber
    \end{align}

Choosing  $u=4\rho \sqrt{\vert \RR \vert} \leq \beta$ and using the fact that $\frac{1-\beta}{1-u}\le 1-\beta+u$ for any $u \leq \beta$, we have

   \begin{align}
        \label{dm-9}
         & \E H^{k+1}  \\
        \leq &   \lp 1-\beta
        + 8\rho \sqrt{\vert \RR \vert} \rp \lp \frac{1}{1-\gamma} \E H^k
        +\frac{4(\alpha^k)^2M^2}{\gamma}\rp  \nonumber \\
        \leq &   \lp 1- \omega_1 \rp \lp \frac{1}{1-\gamma} \E H^k
        +\frac{4(\alpha^k)^2M^2}{\gamma}\rp  , \nonumber
    \end{align}
where $\omega_1 = \beta - 8\rho \sqrt{\vert \RR \vert} >0$. If we set $\frac{1}{1-\gamma}=1+\omega_1$,
     \eqref{dm-9} becomes
    \begin{align}
    \label{dm-10}
    \E H^{k+1}  \le  \lp 1-\omega_1^2\rp \E H^k +   4\omega_2(\alpha^k)^2M^2,
    \end{align}
     where $\omega_2 = \frac{1-\omega_1^2}{\omega_1}$.
We further choose the decaying step size  $\alpha^k= \frac{1}{k+k_0}$, use telescopic cancellation on \eqref{dm-10} from $0$ to $k$, and deduce that
 \begin{align}
 \label{dm-11}
        &\E H^{k}
        \le  (1-\omega_1^2)^k H^0  \\
         &+4\omega_2M^2\lp \frac{1}{(k+k_0-1)^2}+\cdot\cdot\cdot+\frac{(1-\omega_1^2)^{k-1}}{(k_0)^2} \rp . \nonumber
    \end{align}

According to Lemma 5 in \cite{9462519}, there exists some constant $a_1 \geq \frac{(k_0+1)^2}{k_0^2}$ such that
    \begin{align}
    \label{dm-12}
            &\E H^{k}
        \le  (1-\omega_1^2)^k H^0
        +\frac{8 \omega_2 a_1 M^2}{\omega_1^2} \frac{1}{(k+k_0)^2}.
    \end{align}
When all non-Byzantine agents have the same initialization, (\ref{dm-12}) becomes
    \begin{align}
    \label{dm-13}
            &\E H^{k}
        \le
        % \frac{8 \omega_2 a_1 M^2}{\mu^2\omega_1^2}
        \frac{c M^2}{(k+k_0)^2},
    \end{align}
where $c=\frac{8 \omega_2 a_1}{\omega_1^2}$. This completes the proof.
\end{proof}

\end{appendices}

% \newpage
% \bibliographystyle{IEEEtran}
% \bibliography{TSP}

\end{document}